\documentclass[pdflatex,sn-mathphys-num]{sn-jnl}% Math and Physical Sciences Numbered Reference 

\usepackage{graphicx}%
\usepackage{multirow}%
\usepackage{amsmath,amssymb,amsfonts}%
\usepackage{amsthm}%
\usepackage{mathrsfs}%
\usepackage[title]{appendix}%
\usepackage{xcolor}%
\usepackage{textcomp}%
\usepackage{manyfoot}%
\usepackage{booktabs}%
\usepackage{algorithm}%
\usepackage{algorithmicx}%
\usepackage{algpseudocode}%
\usepackage{listings}%
\theoremstyle{thmstyleone}%
\newtheorem{theorem}{Theorem}%  meant for continuous numbers
\newtheorem{lemma}[theorem]{Lemma}

\theoremstyle{thmstyletwo}%

\theoremstyle{thmstylethree}%
\newtheorem{definition}{Definition}%

\begin{document}

\title[Article Title]{Conjunction Subspaces Test for Conformal and Selective Classification}

%%=============================================================%%
%% GivenName	-> \fnm{Joergen W.}
%% Particle	-> \spfx{van der} -> surname prefix
%% FamilyName	-> \sur{Ploeg}
%% Suffix	-> \sfx{IV}
%% \author*[1,2]{\fnm{Joergen W.} \spfx{van der} \sur{Ploeg} 
%%  \sfx{IV}}\email{iauthor@gmail.com}
%%=============================================================%%

\author*[1]{\fnm{Zengyou} \sur{He}}\email{zyhe@dlut.edu.cn}

\author[1]{\fnm{Zerun} \sur{Li}}\email{lizerun2000@163.com}

\author[1]{\fnm{Junjie} \sur{Dong}}\email{jd445@qq.com}

\author[1]{\fnm{Xinying} \sur{Liu}}\email{72317011@mail.dlut.edu.cn}

\author[1]{\fnm{Mudi} \sur{Jiang}}\email{792145962@qq.com}

\author[1]{\fnm{Lianyu} \sur{Hu}}\email{hly4ml@gmail.com}

\affil*[1]{\orgdiv{School of Software}, \orgname{Organization}, \orgaddress{\street{Tuqiang Road}, \city{Dalian}, \postcode{116620}, \state{Liaoning Province}, \country{China}}}

%

%%==================================%%
%% Sample for unstructured abstract %%
%%==================================%%

\abstract{In this paper, we present a new classifier, which integrates significance testing results over different random subspaces to yield consensus \textit{p}-values for quantifying the uncertainty of classification decision. The null hypothesis is that the test sample has no association with the target class on a randomly chosen subspace, and hence the classification problem can be formulated as a problem of testing for the conjunction of hypotheses. The proposed classifier can be easily deployed for the purpose of conformal prediction and selective classification with reject and refine options by simply thresholding the consensus \textit{p}-values. The theoretical analysis on the generalization error bound of the proposed classifier is provided and empirical studies on real data sets are conducted as well to demonstrate its effectiveness.}


%%\pacs[JEL Classification]{D8, H51}

%%\pacs[MSC Classification]{35A01, 65L10, 65L12, 65L20, 65L70}

\maketitle

\section{Introduction}\label{sec1}
Classification is one of the most important tasks in statistics, machine learning and data mining. To date, thousands of classification algorithms (classifiers) have been developed, ranging from simple lazy classifiers such as \textit{k}-nearest neighbor (kNN) \cite{cover1967nearest} to more complex ensemble methods like random forest \cite{breiman2001random}. Despite of the existence of so many classifiers, no single classification algorithm can always beat other algorithms in terms of the classification accuracy \cite{fernandez2014we}. 

Although the classification accuracy is probably one of the most important metrics that we are trying to optimize in practice, other capabilities such as model interpretability, prediction uncertainty quantification and prediction quality control are very vital in diverse applications as well. Therefore, new classifiers which can provide some distinct characteristics that remain unexplored by existing classification methods still need to be developed.

\subsection{Research Motivation}
In this paper, we bridge four domains and introduce a versatile classification algorithm that can alleviate those limitations in corresponding domains and provide several desirable functionalities which are critical to modern applications in an integrated manner. More precisely, our new classifier is closely related to research efforts in the following domains.

(1) \textbf{Classification via significance testing}. The development of classifiers that are capable of solving the classification problem from a significance testing aspect is an overlooked direction in the field of classification.  The linkage between classification and statistical hypothesis testing was first discussed by \cite{liao2007test}. More recently, several testing-based classification methods have been presented by \cite{ghimire2012classification,guo2019interpoint,he2021instance}. However, all these existing testing-based classifiers are instance-based classification methods. That is, these classifiers are lazy learning methods without summarizing the training data to construct a predictive model. Furthermore, the classification accuracy of these methods is still not comparable to those main-stream classifiers such as random forests. Thus, how to construct a non-lazy testing-based classification method with good performance is still an open issue.

(2)	\textbf{Conformal prediction}. Conformal prediction (CP) is a general framework that can be employed to control the error rate of any classifier \cite{shafer2008tutorial}. The basic idea is to predict a set of labels that are guaranteed to include the true label with a high probability. To achieve this objective, a nonconformity measure is utilized to perform a randomness test for each test sample. Based on the given training set, a \textit{p}-value for each possible label is obtained under the null hypothesis that the test sample with its predicted label belongs to the same distribution as the training samples \cite{cherubin2019majority}. To construct the set of predicted labels, the \textit{p}-value for each label is compared to a user-specified threshold. Despite of the success of various CP algorithms and systems, it inherently obtains the \textit{p}-values through a pos-processing procedure on the prediction results of existing classifiers. That is, the significance testing issue and the calculation of \textit{p}-value are not an integral part of the underlying classification algorithm. In addition, the obtained \textit{p}-value is an empirical one, whose precision and uncertainty are highly dependent on the number of training samples. Hence, it would be plausible to have a classification algorithm that can cast the classification problem as a hypothesis testing issue in a straightforward manner and yield an analytical \textit{p}-value for each possible label.

(3)	\textbf{Selective classification}. The term “selective classification” has been defined by \cite{el2010foundations}, which refers to “classification with a rejection option”. Here we expand its meaning to include “a refine option”. That is, selective classification in this paper refers to “classification with both rejection and refine options”. In this scenario, a regular classifier is extended to include the following two additional opinions: one test sample will not be assigned to any class when we resort to a rejection opinion and one test sample can assigned to multiple classes rather than just one class with a refine opinion. These two options are critical to many real applications in which a misclassification is disastrous and too severe to bear \cite{zhang2018reject}. Classifiers that are equipped with one of such two options \cite{chzhen2021set,hendrickx2021machine} or both  (e.g. \cite{zhang2018reject,guan2022prediction}) have been extensively investigated in statistics and machine learning. Most of these existing solutions either adopt the CP framework (e.g. \cite{guan2022prediction}) or need to develop very complicated algorithms to incorporate these additional options. Hence, a classifier with an innate ability to provide these options is highly demanded.

(4)	\textbf{Random projection ensemble classification}. Ensemble classification is a general framework in which multiple weak learners are combined to achieve more stable and accurate classification results. Recently, one particular ensemble classification strategy has received much attention, in which each base classifier is trained on a projection (or subspace) optimally selected from a set of random projections \cite{cannings2017random,tian2021rase}. Similar to other popular ensemble classification approaches, the final classification decision is mainly based on majority vote in these methods. As a result, it is a non-trivial task to employ such random projection ensemble classifiers to carry out conformal and selective prediction. 

\subsection{Method Outline and Contributions}
In this paper, we present a new classifier based on random subspace ensemble and hypothesis testing, which works as follows. First of all, we repeatedly generate a set of random feature subspaces and choose the best one from this set according to a selection criterion. Secondly, in each subspace obtained in the previous step, we construct a significance-based classifier in which the classification problem is formulated as a statistical association testing issue. The null hypothesis is that the test sample has no association with the target class and hence a smaller \textit{p}-value would indicate that the test sample is likely to belong to the corresponding class. Thirdly, we merge \textit{p}-values on all subspaces via meta-analysis \cite{borenstein2021introduction} to obtain a consensus \textit{p}-value for each class. Here we employ the rOP (\textbf{r}-th \textbf{o}rdered \textit{\textbf{p}}-value) method \cite{song2014hypothesis} to fulfill the \textit{p}-value combination task and select a best “r” during the training phase. The final consensus \textit{p}-value for each class can be utilized to facilitate the classification decision such that class labels with smaller \textit{p}-values are preferred to those labels with larger \textit{p}-values.

The new classifier presented in this paper is named as COST (\textbf{Co}njunction \textbf{S}ubspaces \textbf{T}est), which has the following appealing features. Firstly, it is a testing-based classifier so that it can provide \textit{p}-values for prediction uncertainty quantification in a natural manner. We don’t need to post-process the prediction results of third-part classification algorithms. Secondly, thanks to the good prediction performance of the random subspace ensemble classification framework, it can achieve very good classification accuracy that is comparable to off-the-shelf classifiers such as random forests even though its base learner is only built on statistical association testing. Finally, it can be easily deployed for the purpose of selective classification by setting a significance threshold. That is, we resort to a reject option for the test sample if its \textit{p}-values of all classes exceed the significance threshold. Similarly, the refine option would be adopted if the \textit{p}-values of more than one class are less than the significance threshold.

Overall, the main contributions of this paper can be summarized as follows:
\begin{itemize}
	\item From the testing-based classification aspect, our COST algorithm is the first non-lazy classifier towards this direction. More importantly, it is able to achieve comparable classification accuracy to those state-of-the-art classifiers.
	\item From the conformal prediction aspect, our algorithm provides an alternative way to obtain \textit{p}-values for uncertainty qualification without the need to post-process the outputs of existing classifiers. 
	\item From the selective classification aspect, the reject and refine option can be included as the by-product of our algorithm in a natural and elegant manner.
	\item From the random subspace ensemble classification aspect, both the base classifier construction and the ensemble strategy are not investigated in previous studies.
\end{itemize}

The remaining parts of this paper are organized as follows. In Section \ref{Related-work}, classification methods that are closely related to our algorithm are discussed. In Section \ref{Method} and Section \ref{theoretical-analysis}, the details of presented COST algorithm are described. In Section \ref{experiments}, empirical results on real data sets are presented. In Section \ref{conclusion}, some discussions and conclusions are given.

\section{Related Work}\label{Related-work}

\subsection{Classification via Hypothesis Testing}
Binary classification and hypothesis testing are largely regarded as two separate research topics \cite{li2020statistical}. As a result, their connection and differences have been rarely discussed.  \cite{li2020statistical} summarized and compared these topics for the broad scientific community. 

\cite{liao2007test} proposed a classification algorithm based on hypothesis testing in the two-class setting. If the test sample is placed into the wrong class, then the difference between the two classes will be blurred. Based on this idea, two tests with respect to the equality of two means are conducted. In each significance test, the test sample with an unknown class label is assumed to belong to one of the two classes. Accordingly, we will obtain two \textit{p}-values and the test sample is assigned to the class that has a smaller \textit{p}-value. \cite{ghimire2012classification} further extended the method of \cite{liao2007test} by introducing a minimum distance into the classifier for image pixels.

\cite{modarres2014interpoint,modarres2016multivariate,modarres2018multinomial} investigated the properties of squared Euclidean interpoint distances (IPDs) among different samples taken from multivariate Bernoulli, multivariate Poisson and multinomial distributions. Afterwards, a new testing-based classifier is proposed based on the IPDs among different samples by \cite{guo2019interpoint}.

\cite{he2021instance} formulated the binary classification problem as a two-sample testing problem. More precisely, the method first calculates the distance between the test sample and each training sample to derive two distance sets. Then, the two-sample testing method called Wilcoxon-Mann-Whitney test is performed under the null hypothesis that the two sets of distances are drawn from the same cumulative distribution. At last, two \textit{p}-values are generated and the test sample is assigned to the class associated with the smaller \textit{p}-value.

Overall, these existing classification algorithms based on significance testing still deserve certain limitations. Firstly, they are instance-based classification algorithms without training a concise predictive model. Secondly, the classification accuracies of these methods are still far from satisfactory in comparison with those state-of-the-art classifiers. Finally, these methods are time-consuming since the pairwise distance between the test sample and each training sample is generally required.

\subsection{Conformal Prediction}

The concept of conformal prediction was first introduced by \cite{vovk2005algorithmic}. Thereafter, extensive research efforts have been pursued under this framework and significant advances have been made. Due to the vast literature during the past decades, please refer to a recent introduction \cite{toccaceli2022introduction} on this topic to get a global view on recent advances. Here we only focus on the key difference between the CP method and our method and discuss those ensemble CP methods. 

There are at least two critical differences between the CP method and our algorithm. First, CP operates on top of virtually existing classification methods to report empirical \textit{p}-values for subsequent analysis. In contrast, our method inherently utilizes the significance testing procedure as an integral component to generate analytical \textit{p}-values as outputs. Second, the null hypotheses are different, i.e., the null hypothesis in CP is that the test sample belongs to the same distribution as the training samples of a candidate class while the null hypothesis in our method is that the test sample has no association with the candidate class.

Very recently, the combination of multiple conformal predictors has received much attention \cite{toccaceli2017combination,toccaceli2019conformal,toccaceli2019combination,linusson2020efficient}. These methods are closely related to our algorithm since the integration of multiple \textit{p}-values via meta-analysis is typically adopted in some methods as well. Anyway, none of these methods has employed the rOP approach in the ensemble stage. In addition, the selection of an optimal parameter “\textit{r}” during the training phase to obtain better classification accuracy is another merit of our algorithm. 

\subsection{Selective Classification}

Both classifiers with a reject option \cite{Magesh2023JMLR,hendrickx2021machine,el2010foundations,yuan2010classification} and set-valued classifiers (classifier with a refine option) \cite{chzhen2021set,wang2022set,bates2021distribution} have been extensively studied during the past decades. See \cite{hendrickx2021machine} and \cite{chzhen2021set} for an overview on the advances of these two topics. Meanwhile, we have witnessed an increased interest on developing selective classifiers in which both reject option and refine option are equipped \cite{zhou2023JMLR,zhang2018reject,guan2022prediction}.

From the selective classification perspective, our method also provides both reject option and refine option in the same classifier. Note that CP methods can achieve this objective in a similar manner. The main difference between our method and CP methods have been discussed in Section 2.2. Here we highlight the fact that the CP methods typically will not resort to the reject option since it has to guarantee the probability of including the true label in the prediction set. In the experimental results, we will further discuss this potential limitation of the CP method. Compared to those non-CP methods such as \cite{zhang2018reject}, our method is able to provide both options solely based on comparing \textit{p}-values with a given significance threshold. Moreover, our method is quite simple and easy-to-understand in comparison to those selective classification algorithms without adopting the CP framework.

\subsection{Random Projection Ensemble Classification}

The random projection ensemble classification method combines classification results of from base classifiers on random projections of original feature vectors \cite{cannings2017random}. From the dimension reduction viewpoint, the random subspace ensemble classifier \cite{ho1998random,tian2021rase,huynh2023optimizing} can be regarded as a special case of the random projection ensemble classifier. 

Our algorithm is one special type of random subspace ensemble classifier as well. However, there are at least two key differences between our method and existing classifiers \cite{tian2021rase,huynh2023optimizing} with respect to the base learner and the ensemble strategy. In general, any existing classifier can be employed as the base learner on each subspace, our method utilizes a quite simple classifier based on association testing that has never been discussed before. In addition, due to the special characteristics of testing-based base learners in our method, the \textit{p}-value combination method is used to integrate multiple classification results from different base learners. Furthermore, to achieve the objective of selective classification using existing random subspace ensemble classifiers, we may have to follow the CP framework to post-process the original prediction results.

\section{Method}\label{Method}
\subsection{Notations}
In this section, we will introduce the details of our COST algorithm. Assume that we have a training set of $n$ labeled samples, denoted as $(x_{1},y_{1}),...,(x_{n},y_{n})$, where each $x_i$ represents a training sample and $y_i$ represents its corresponding label. In addition to these training samples, we also consider a test sample $\hat{x}$ of $d$ feature values, whose true class label is denoted by $\hat{y}$. Other symbols used in the COST algorithm and their meanings are presented in Table \ref{table1}. Note that the COST algorithm is specifically developed for categorical data. To handle continuous data, a discretization process can be utilized for data transformation. 
\begin{table}[h]
	\centering
	\caption{Notations.}
	\begin{tabular}{ccccc} \toprule
		Notation         &Meaning                          &  \\ 
		\midrule
		$\alpha$                & The significance level        &  \\
		$b_1$                &Number of chosen subspaces        &  \\
        $b_2$               & Number of candidate subspaces in each round & \\
		$n$                &Number of training samples         &  \\
        $m$                 &Number of validation samples       & \\
        $l_{max}$           & The maximum size of subspace & \\
	    $F=(a_1,\dots,a_d)$                & $d$ features                    &    \\ 
		$(c_1,\dots,c_k)$                & $k$ class labels                    &  \\ 
        $(s_1,\dots,s_{b_1})$          & $b_1$ chosen subspaces & \\
        $CS=(cs_1,\dots,cs_{b_2})$     & $b_2$ candidate subspaces in each iteration & \\
        $p_{i,j},1\leq i\leq b_1,1\leq j\leq k$ & The \textit{p}-value of $s_i$ on $c_j$ & \\
        $p_{c_i},1\leq i \leq k$        & The \textit{p}-value of $\hat{x}$ on $c_i$ & \\
        $D_t=\{(x_i,y_i)  | 1 \leq i \leq n\}$     & Training data      & \\
        $D_v=\{(x_i,y_i)  | 1 \leq i \leq m\}$      & Validation data   & \\
		$(\hat{x},\hat{y})$ & Test sample with its true label              &  \\
		\bottomrule
	\end{tabular}
	\label{table1}
\end{table}

\subsection{An Overview of COST}
An overview of the COST algorithm is shown in Figure \ref{figure1}. As shown in Figure \ref{figure1}: COST is composed of three key steps: subspace selection, \textit{p}-value calculation and \textit{p}-value combination. 

\begin{itemize}
	\item \textbf{Subspace selection}. Within the COST framework, each iteration yields $b_2$ random subspaces. Then, among these $b_{2}$ subspaces, one subspace with the highest relative risk is identified and selected. After repeating this procedure $b_1$ times, $b_{1}$ subspaces are chosen for training the classifier.
	\item \textbf{\textit{P}-value calculation}. With respect to each chosen subspace, the null hypothesis is that the test sample has no association with each class. The alternative hypothesis is that the test sample is positively associated with one target class in the sense that the projection of feature values of the test sample on this subspace is over-expressed in the target class. The Fisher's exact test can be employed to tackle the above hypothesis testing issue to yield a \textit{p}-value for each class with respect to each subspace.  
	\item \textbf{\textit{P}-value combination}. We have collected multiple \textit{p}-values with respect to different subspaces for each class. Under the joint null hypothesis that the test sample has no association with the target class, we can employ the \textit{p}-value combination methods in meta-analysis to obtain a consensus \textit{p}-value. According to the consensus \textit{p}-values for different classes, we can conduct a regular classifier by assigning the test sample to the class with the smallest \textit{p}-value. To include the reject option, we can just specify a significance threshold to “refuse to classification" if the \textit{p}-values of all classes are larger than the threshold. In a similar manner, we can include the refine option by classifying the test sample to more than one class if their corresponding \textit{p}-values are no larger than the threshold. 
\end{itemize}

\begin{figure}[h]
	\centering
	\includegraphics[width=\columnwidth]{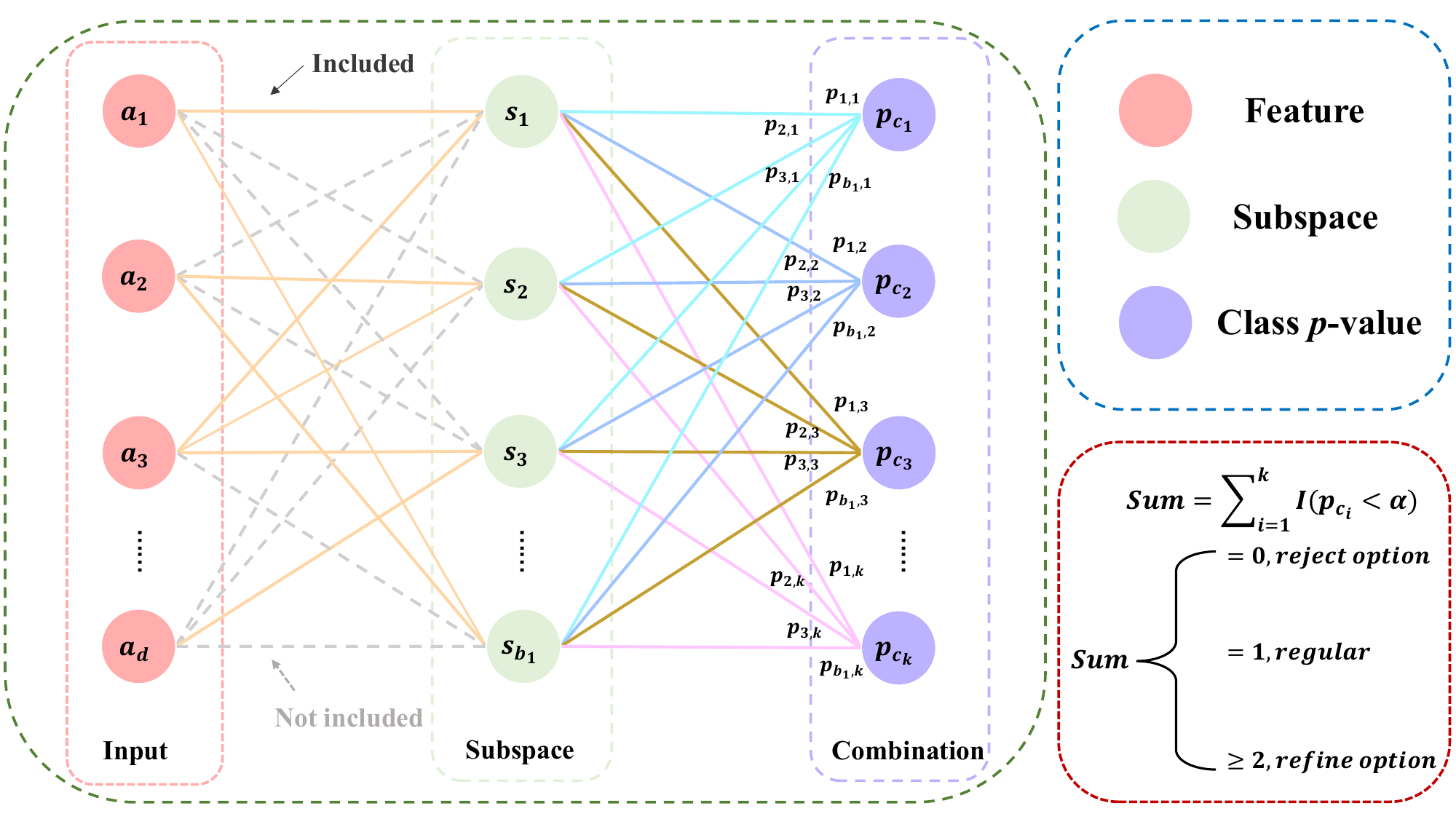}
	\caption{\textbf{An overview of COST}. Firstly, a specified number of subspaces are chosen according to some criteria. Secondly, a \textit{p}-value is calculated for each class on every subspace. Finally, the \textit{p}-value combination method is used to merge \textit{p}-values from all subspaces to obtain a consensus \textit{p}-value for each class. Through the comparison among these \textit{k} class \textit{p}-values or the introduction of a significance level,  we can fulfill the task of regular classification or selective classification. }
	\label{figure1}
\end{figure}

\subsection{Subspaces Selection}
\label{sec:subspace-selection}
In each iteration, COST randomly generates $b_2$ random feature subsets of $F=\{a_1,\dots,a_d\}$, which are denoted by $CS=\{cs_1,\dots,cs_{b_2}\}$.  Each subset in $CS$ corresponds to a respective candidate subspace. Furthermore, for any $cs \in CS $, we impose a constraint on its size, i.e.,  $1 \leq |cs| \leq l_{max}$. Here, $l_{max}$ represents the maximum size of subspaces, which is specified as $l_{max}=min(d,\lfloor \sqrt{n}\rfloor)$, where $d$ denotes the number of total features and  $n$ represents the number of training samples.

Next, we need to select one “best" subspace from these $b_2$ candidate subspaces according to some criteria. For a candidate subspace $cs$, we use $x_{l}[cs]$ to denote the projected feature value set from the \textit{l}-th training sample. The collection of distinct projected feature value sets of $n$ training samples on the subspace $cs$ is denoted by $Z$. 

For each unique feature value set $z \in Z$, the number of its occurrences in the training set can be calculated as  $o(z)=| \{x_{l} | x_{l}[cs]=z, 1\leq l\leq n\}|$. Similarly, the number of its occurrences in the training set with respect to the \textit{h}-th class is $o_{h}(z)=| \{x_{l}\ | \ x_{l}[cs]=z, y_{l}=c_{h}, 1 \leq l\leq n, 1\leq h \leq k\}|$. Accordingly, we can find a class that $z$ appears most frequently, which is denoted by $c_{t}$. Then, we can calculate the relative risk of $z$ as:
  \begin{equation}
  	\label{math:single-RR}
	RR(z)=\frac{o_{t}(z)/o(z)}{(o_{t}-o_{t}(z))/(n-o(z))},
 \end{equation}
where $o_{t}$ is number of training samples whose class label is $c_{t}$. By averaging over all relative risk values of distinct projected feature value sets, we have:  
   \begin{equation}
   	\label{math:average RR}
 	\overline{RR} = \sum_{z \in Z}^{}RR(z)  / |Z|.
 \end{equation}
 
According to Equation (\ref{math:average RR}), we can calculate the average relative risk for all $b_2$ candidate subspaces and retain the one with the highest average relative risk. After iterating $b_1$ times, we obtain $b_1$ subspaces.
\subsection{\textit{P}-value Calculation}
\label{sec:pvalue-calculation}
In each individual subspace, we need to calculate the \textit{p}-value corresponding to each class. Suppose we are computing the \textit{p}-value $p_{i,j}$ on the subspace $s_i$ for class $c_j$. The corresponding null hypothesis in this case is that, the test sample $\hat{x}$ is hypothesized not to belong to $c_j$ according to projected training samples on the subspace $s_i$. More precisely, the null hypothesis can be further stated as: “the test sample has no association with $c_{j}$ on $s_{i}$". The COST algorithm tackles the above statistical association testing issue using the Fisher's exact test to calculate an analytical \textit{p}-value.

Similar to the notations used in Section \ref{sec:subspace-selection}, let $z_{i} = \hat{x}[s_{i}]$ be projected feature value set of $\hat{x}$ on the subspace $s_{i}$. Then, $o(z_{i})$ and $o_{j}(z_{i})$ are used to denote the number of occurrences of $z_{i}$ in all training samples and the training samples of class $c_{j}$, respectively. Meanwhile, $o_{j}$ is used to represent the number of training samples whose class label is $c_{j}$.

Then, we can construct a contingency table for two binary variables defined below. One is the class variable denoted by $C$, where $C=1$ if the class label of training samples is $c_j$ and $C=0$ otherwise. Another variable is denoted by $E$, where $E=1$ if the projected feature value set of training samples on $s_{i}$ is $z_{i}$ and $E=0$ otherwise.

\begin{table}[h]
	\centering
	\caption{The 2*2 contingency table for computing $p_{i,j}$.}
	\begin{tabular}{cccccc} \toprule
		  &     $C=1$    &$C=0$    & sum    &  \\ 
		\midrule
	  $E=1$ 	 &    $o_{j}(z_{i})$   &$o(z_{i}) - o_{j}(z_{i})$   & $o(z_i)$    &  \\ 
        $E=0$ 	 &    $o_j-o_{j}(z_{i})$   &$n-o_j-o(z_{i})+o_{j}(z_{i})$   & $n-o(z_{i})$  &  \\ 
        sum 	 &    $o_j$   &  $n-o_j$  & $n$  &  \\ 
		\bottomrule
	\end{tabular}
	\label{table:fisher}
\end{table}

Under the null hypothesis that the test sample $\hat{x}$ is not associated with the class  $c_j$ on the subspace $s_i$, the cell count $o_{j}(z_{i})$ follows a hypergeometric distribution:
\begin{equation}
	\label{math:hypergeometric}
	P(o_{j}(z_{i})|s_i)=\frac{\binom{o_j}{o_{j}(z_{i})} \binom{n-o_{j}(z_{i})}{o(z_{i})-o_{j}(z_{i})}}{\binom{n}{o(z_{i})}}.
\end{equation}
The \textit{p}-value represents the likelihood of observing results that are as extreme as, or more extreme than the results we have obtained under the assumption that the null hypothesis is true. Within the context of our study, “more extreme" refers to instances that the cell count $C = 1$ and $E = 1$ exceeds $o_{j}(z_{i})$. It means that $z_{i}$ appears more frequently in the class $c_{j}$ than expected. Therefore, the \textit{p}-value $p_{i,j}$ can be computed as the cumulative probability:
\begin{equation}
	\label{math:p-value}
	p_{i,j} = p(o_{j}(z_{i})|s_i)=\sum_{q=0}^{min(o(z_{i}) - o_{j}(z_{i}),o_j-o_{j}(z_{i}))} \ P(o_{j}(z_{i})+q|s_i),
\end{equation}
where $min(o(z_{i}) - o_{j}(z_{i}),o_j-o_{j}(z_{i}))$ denotes the maximal number of samples of class $c_{j}$ that  are able to contain $z_{i}$ beyond the exsiting $o_{j}(z_{i})$ samples.

\subsection{\textit{P}-value Combination}
\label{sec:p-value-combination}
For the test sample $\hat{x}$, we have calculated the \textit{p}-values $p_{i,j} (1\leq i\leq b_1,1\leq j\leq k$) on all subspaces for all classes in Section \ref{sec:pvalue-calculation}. In this section, we combine the \textit{p}-values $\{p_{1,j},\dots,p_{b_1,j}\}$ to obtain a consensus \textit{p}-value $p_{c_j}$ via meta-analysis.

In this paper, the meta-analysis method used for combining the \textit{p}-values is the \textit{r}OP (\textbf{\textit{r}}-th \textbf{o}rdered \textit{\textbf{p}}-value) method.  In this method, we first sort $b_{1}$ \textit{p}-values in a non-decreasing order: $p_{(1),j}\leq\dots\leq p_{(b_1),j}$ and then choose the \textit{r}-th \textit{p}-value $p_{(r),j}$ as the test statistic. According to \cite{song2014hypothesis}, the  $p_{(r),j}$ follows a Beta distribution  with parameters $ \alpha = r$ and $\beta = b_1-r+1$. Hence, we can obtain the consensus \textit{p}-value $p_{c_{j}}$ according to the cumulative distribution function of the Beta distribution.

We can assign the test sample $\hat{x}$ to the class with the minimum \textit{p}-value to conduct a regular classification:
\begin{equation}
	\label{math:regular_ans}
	\hat{y}= \underset{c_i,1 \leq i \leq k}{arg \ min} \ p_{c_i}.
\end{equation}

Within the COST framework, the selection of a proper “\textit{r}” is crucial to the final classification performance. Hence, a validation set $D_v=\{(x_i,y_i)  | 1 \leq i \leq m\}$ is utilized to choose the optimal value of “\textit{r}”. For each potential value of “\textit{r}”, the classification accuracy on $D_v$ is computed. Consequently, the “\textit{r}” that yields the highest classification accuracy is chosen as the final parameter setting to be applied to the test set.

\subsection{Selective Classification}
\label{section:selective classification}
Following the aforementioned procedures, we are able to ascertain the \textit{p}-values for the test sample $\hat{x}$ for all classes $p_{c_1},\dots,p_{c_k}$. We calculate the number of \textit{p}-values that are less than a given significance level $\alpha$:
\begin{equation}
	\label{math:indicator}
	Sum =   \sum_{i=1}^{k}\ \mathbb{I}(p_{c_i} \ < \alpha),
\end{equation}
where $\mathbb{I}$ is an indicator function and $\alpha$ is a user-specified parameter. 

Based on the value of $Sum$, we have the following classification options:
\begin{itemize}
	\item $Sum \ = \ 1$, this indicates that exactly only one \textit{p}-value can pass the threshold. The test sample $\hat{x}$ is to be assigned distinctly to the class whose associated \textit{p}-value falls below the predetermined threshold $\alpha$.
	 %$\hat{y} = \{ c_i \; | \; p_{c_i} < \alpha, \; i = 1, 2, \ldots, k \}$.
	\item $Sum\ =\ 0$, this indicates that all $k$ \textit{p}-values are no less than the threshold. Therefore, it is concluded that the test sample $\hat{x}$ does not belong to any class, and the \textbf{reject option} is opted.
%	This assignment can be formally represented as $\hat{y} = \{-1\}$, where $-1$ signifies rejection of classification, indicating that $\hat{x}$ does not belong to any class.
	\item $Sum \ >\ 1$, this indicates that more than one \textit{p}-values are less than the threshold. Hence, the test sample $\hat{x}$ could be considered to belong to at least two classes with a \textbf{refine option}. 
%	This assignment can be formally represented as $\hat{y} = \{ c_i \; | \; p_{c_i} < \alpha, \; i = 1, 2, \ldots, k \}$. 
\end{itemize}

\section{Theoretical Analysis}\label{theoretical-analysis}
COST is an ensemble classification algorithm based on order statistics. In this section, we will first investigate the convergence of COST under the setting of regular classification, i.e., each test sample is assigned to the class with the smallest consensus \textit{p}-value. 

Suppose $\hat{c}$ is the class label with the minimum \textit{p}-value when the optimal chosen \textit{r} is $\hat{r}$. Accordingly, the \textit{p}-value that ranks at the position $\hat{r}$ in the ordered \textit{p}-value sequence for class $\hat{c}$ is $p_{(\hat{r}),\hat{c}}$. As shown in the following Lemma, the classification decision of COST is equivalent to that of a bagging predictor.

\begin{lemma}
\label{lemma:bagging-predictor}
On the i-th subspace $(1\leq i\leq b_1)$, we can construct a threshold classifier as: $f(\hat{x},s_i)=c_{j} \ if\  \mathbb{I}(p_{i,j} \leq p_{(\hat{r}),\hat{c}})$. Then, these $b_1$ threshold classifiers can be aggregated to form a bagging predictor:
\begin{equation}
    g(\hat{x}) =  \underset{c_j,1 \leq j \leq k}{arg \ max}\sum_{i=1}^{b_1} \mathbb{I}(f(\hat{x},s_i)=c_{j}).
\end{equation}
The class label reported by $g(\hat{x})$ is equivalent to that of COST.
\end{lemma}

\begin{proof}
According to the \textit{r}OP method and COST, we know that for the class $\hat{c}$, there are at least $\hat{r}$ \textit{p}-values that are no larger than $p_{(\hat{r}),\hat{c}}$. That is, there are at least $\hat{r}$ threshold classifiers that will vote $\hat{c}$. Meanwhile, since $p_{(\hat{r}),\hat{c}}$ is the smallest \textit{p}-value among all $p_{(\hat{r}),j}(1\leq j\leq k)$, without loss of the generality, we can assume that all other $p_{(\hat{r}),j}$ are strictly less than  $p_{(\hat{r}),\hat{c}}$. Hence, there will be at most $\hat{r}-1$ threshold classifiers that support the remaining $k-1$ classes. As a result, the bagging predictor $g(\hat{x})$ will classify $x$ to $\hat{c}$ since it will receive the largest number of votes from base classifiers.  So the lemma is proved.
\end{proof}

\subsection{The Convergence Analysis}
After establishing the equivalence between COST and the bagging predictor $g(\hat{x})$ defined in Lemma \ref{lemma:bagging-predictor} with respect to the classification result, we can follow the notation in \cite{breiman2001random} to define a marginal function as follows.

\begin{definition}
\label{definition-margin-function}
Given an ensemble classifier of $b_1$ base predictors $f(\hat{x},s_1),f(\hat{x},s_2),...,f(\hat{x},s_{b_1})$, we define the margin function of COST for the input sample $\hat{x}$ as follows:
\begin{equation}
    MF(\hat{x},\hat{c})=\frac{1}{b_1}\sum_{i=1}^{b_1}\mathbb{I}(f(\hat{x},s_i)=\hat{c})-\underset{c_j\neq \hat{c}}{max}\frac{1}{b_1}\sum_{i=1}^{b_1}\mathbb{I}(f(\hat{x},s_i)=c_j),
\end{equation}
where $\mathbb{I}$ is an indicator function.
\end{definition}

In Definition \ref{definition-margin-function}, if $\hat{c}=\hat{y}$ is supposed to the right class of test sample $\hat{x}$, then the margin function measures the extent to which the average number of votes for $\hat{c}$ exceeds the average vote of the second best class.
\begin{lemma}
\label{Lemma:converge}
    As the number of base classifiers increases $(b_1\to \infty)$ and these base classifiers are supposed to be independent, for almost surely all i.i.d random subspaces $s_{1}, s_{2},...,$ the margin function $MF(\hat{x},\hat{c})$ at $\hat{x}$ converges to
    \begin{equation}
        MF^*(\hat{x},\hat{c})=P_S(f(\hat{x},S)=\hat{c})-\underset{c_j\neq \hat{c}}{max}P_S(f(\hat{x},S)=c_j),
    \end{equation}
    where $S$ is used to denote the random subspace variable.  
\end{lemma}
\begin{proof}
It is sufficient to show that for $s_1,s_2,...,$ and for all $\hat{x}$,
\begin{equation}
    \frac{1}{b_1}\sum_{i=1}^{b_1}\mathbb{I}(f(\hat{x},s_i)=c_j) \xrightarrow{b_1 \to \infty}  P_S(f(\hat{x},S)=c_j),
\end{equation}
where $S$ is used to denote the random subspace variable.  

For a fixed training set and the subspace $S$, the set of all $\hat{x}$ such that $f(\hat{x}, S) = c_{j}$ is a subset of distinct projected feature value sets over $S$ on the training set. Since both the number of all possible feature subsets is finite and the number of all possible feature value sets is finite, for all  $f(\hat{x}, S)$, there is only a finite number $T$ of such subsets of projected feature value sets, which are denoted by $H_1, H_2,.....H_{T}$.  A function $\varphi(S)=t$ is thus defined if $\{\hat{x}:f(\hat{x},S)=c_j\}=H_t$. If the frequency of $\varphi(s_i)=t$ in the initial $b_1$ subspaces is denoted by $N_t$, then we have:
\begin{equation}
    \frac{1}{b_1}\sum_{i=1}^{b_1}\mathbb{I}(f(\hat{x},s_i)=c_j)=\frac{1}{b_1}\sum_{t}^{}N_t\mathbb{I}(\hat{x}\in H_t).
\end{equation}

According to the Law of Large Numbers, as the number of base classifiers $b_1$ increases,
\begin{equation}
    N_t=\frac{1}{b_1}\sum_{i=1}^{b_1}\mathbb{I}(\varphi(s_i)=t),
\end{equation}
converges almost surely with probability 1 to 
\begin{equation}
    E_S[\mathbb{I}(\varphi(s)=t)]=P_S(\varphi(S)=t).
\end{equation}
Thus, 
\begin{equation}
    \label{equation-converge-to-P_s}
    \frac{1}{b_1}\sum_{i=1}^{b_1}\mathbb{I}(f(\hat{x},s_i)=c_j)\rightarrow \sum_{t}^{}P_S(\varphi(S)=t)\mathbb{I}(x\in H_t).
\end{equation}

Note that the right expression on the right-hand side of Equation (\ref{equation-converge-to-P_s}) is $P_S(f(\hat{x},S)=c_j)$, then the lemma is proved.
\end{proof}

\subsection{The Upper Bound of Generalization Error}

\begin{definition}
\label{definition-GE}
The generalization error is defined as $\varepsilon = P_{X,Y}(MF(\hat{x},\hat{c})<0)$, where the subscripts $X, Y$ indicate that the probability is over the sample space $X$ and the class label space $Y$.
\end{definition}

\begin{lemma}
As the number of base classifiers increases and these classifiers are independent, the generalization error $\varepsilon$ converges to
\begin{equation}
    \varepsilon^*=P_{X,Y}(MF^*(x,c)<0).
\end{equation}
\end{lemma}
\begin{proof}
According to Lemma \ref{Lemma:converge}, the margin function $MF(\hat{x},\hat{c})$ converges to $MF^*(\hat{x},\hat{c})$ when $b_{1}$ classifiers are independent and $b_{1}\to \infty$. Therefore, the generalization error $\varepsilon$ converges to $\varepsilon^*=P_{X,Y}(MF^*(\hat{x},\hat{c})<0)$.
\end{proof}
\begin{definition}
Similar to \cite{breiman2001random}, the strength of COST is defined as
\begin{equation}
    \mathbb{S}=E_{X,Y}MF^*(\hat{x},\hat{c}),
\end{equation}
where the $E_{X,Y}$ represents the expectation in the $(X,Y)$ space. Strength can be regarded as a measure for characterizing the average accuracy of an individual base classifier.
\end{definition}

We define $J(\hat{x},\hat{c})$ as the class $c_j\neq \hat{c}$ that maximizes $P_S(f(\hat{x},S)=c_j)$, which represents the class that are most prone to be misclassified by base classifiers given the input $\hat{x}$. Subsequently, we can introduce a raw margin function for each base classifier with respect to the input sample $\hat{x}$ as follows.
\begin{definition}
\label{definition-raw-margin-function}
The raw margin function is defined as:
      \begin{equation}
        \mathbb{R}(S)=\mathbb{I}(f(\hat{x},S)=\hat{c})-\mathbb{I}(f(\hat{x},S)=J(\hat{x},\hat{c})).
        \end{equation}
\end{definition}
        Based on Definition \ref{definition-raw-margin-function}, we have:
        \begin{equation}
            \label{equation:MF=ES}
            MF^*(\hat{x},\hat{c})=E_S[\mathbb{I}(f(\hat{x},S)=\hat{c})-\mathbb{I}(f(\hat{x},S)=J(\hat{x},\hat{c}))]=E_S\mathbb{R}(S).
        \end{equation}

\begin{theorem}
The generalization error of the COST algorithm has the following upper bound
\begin{equation}
    \varepsilon^*\leq \frac{\bar{\rho}(1-\mathbb{S}^2)}{\mathbb{S}^2},
\end{equation}
where $\bar{\rho}=E_{S,S'}[\rho_{X,Y}(S,S')]$ is the average correlation between two base classifiers by computing the expectation over all pairs of random subspaces $S$ and $S'$. 
\end{theorem}
\begin{proof}
According to Lemma \ref{Lemma:converge}, we have:
\begin{equation}
    MF^*(\hat{x},\hat{c})=P_S(f(\hat{x},S)=\hat{c})-\underset{c_j\neq \hat{c}}{max}P_S(f(\hat{x},S)=c_j).
\end{equation}
If we assume that the strength $\mathbb{S}=E_{X,Y}MF^*(\hat{x},\hat{c})$ is larger than $0$, then we can apply the Chebyshev's inequality to obtain:
\begin{align}
    \varepsilon^*&=P_{X,Y}[MF^*(\hat{x},\hat{c})<0]\nonumber \\
    &\leq P_{X,Y}[\mathbb{S}-MF^*(\hat{x},\hat{c})\ge \mathbb{S}]\nonumber \\
    &=P_{X,Y}[|MF^*(\hat{x},\hat{c})-\mathbb{S}|\ge \mathbb{S}]\nonumber \\
    &\leq \frac{Var_{X,Y}(MF^*(\hat{x},\hat{c}))}{\mathbb{S}^2}.
\end{align}

For any function $f$, we have:
\begin{equation}
    E_S[f(S)]^2=E_{S,S'}[f(S)f(S')],
\end{equation}
where $S,S'$ are two i.i.d. random variables. According to Equation (\ref{equation:MF=ES}), we know that $MF^*(\hat{x},\hat{c})=E_S\mathbb{R}(S)$, then we have:
\begin{equation}
    [MF^*(\hat{x},\hat{c})]^2=[E_S\mathbb{R}(S)]^2=E_{S,S'}[\mathbb{R}(S)\mathbb{R}(S')].
\end{equation}
Consequently, we can compute the variance as follows:
\begin{align}
    Var_{X,Y}(MF^*(\hat{x},\hat{c}))&=E_{X,Y}[[MF^*(\hat{x},\hat{c}]]^2)-[E_{X,Y}(MF^*(\hat{x},\hat{c}))]^2 \nonumber \\
    &=E_{X,Y}[E_{S,S'}[\mathbb{R}(S)\mathbb{R}(S')]] - [E_{X,Y}(E_S\mathbb{R}(S))]^2 \nonumber \\
    &=E_{X,Y}[E_{S,S'}[\mathbb{R}(S)\mathbb{R}(S')]] - [E_{S}(E_{X,Y}\mathbb{R}(S))]^2 \nonumber \\
    &=E_{S,S'}[E_{X,Y}[\mathbb{R}(S)\mathbb{R}(S')]] - E_{S,S'}[E_{X,Y}\mathbb{R}(S)E_{X,Y}\mathbb{R}(S')] \nonumber \\
    &=E_{S,S'}[Cov_{X,Y}(\mathbb{R}(S)\mathbb{R}(S'))] \nonumber \\
    &=E_{S,S'}[\rho_{X,Y}(S,S')\sigma_{X,Y}(\mathbb{R}(S))\sigma_{X,Y}(\mathbb{R}(S'))]\nonumber \\
    &=\bar{\rho}[E_S(\sigma_{X,Y}(\mathbb{R}(S)))]^2,
\end{align}
where $\bar{\rho}=E_{S,S'}[\rho_{X,Y}(S,S')]$ and $\sigma_{X,Y}(\mathbb{R}(S))$ is the standard deviation of $\mathbb{R}(S)$.

We know that for a given random variable $Z$, the variance $Var(Z)$ is non-negative, which implies that $[E(Z)]^2\leq E(Z^2)$, then we have:
\begin{align}
    Var_{X,Y}(MF^*(\hat{x},\hat{c}))&=E_{S,S'}[\rho_{X,Y}(S,S')\sigma_{X,Y}(\mathbb{R}(S))\sigma_{X,Y}(\mathbb{R}(S'))] \nonumber \\
    &=\bar{\rho}[E_S(\sigma_{X,Y}(\mathbb{R}(S)))]^2 \nonumber \\
    &\leq \bar{\rho} E_S(\sigma_{X,Y}(\mathbb{R}(S))^2) \nonumber \\
    &=\bar{\rho}E_S(Var_{X,Y}(\mathbb{R}(S))),
\end{align}
and we know that
\begin{align}
    E_S(Var_{X,Y}(\mathbb{R}(S)))&=E_S[E_{X,Y}[\mathbb{R}(S)^2]-E_{X,Y}[\mathbb{R}(S)]^2] \nonumber \\
    &=E_S[E_{X,Y}[\mathbb{R}(S)^2]]-E_S[E_{X,Y}\mathbb{R}(S)]^2] \nonumber \\
    &\leq E_S[E_{X,Y}[\mathbb{R}(S)^2]]-[E_S(E_{X,Y}[\mathbb{R}(S)])]^2 \nonumber \\
    &=E_S[E_{X,Y}[\mathbb{R}(S)^2]]-[E_{X,Y}(E_{S}[\mathbb{R}(S)])]^2 \nonumber \\
    &=E_S[E_{X,Y}[\mathbb{R}(S)^2]]-[E_{X,Y}MF^*(\hat{x},\hat{c})]^2 \nonumber \\
    &\leq 1-\mathbb{S}^2.
\end{align}
Hence, we can derive the following upper bound:
\begin{equation}
    \varepsilon^*\leq \frac{Var_{X,Y}(MF^*(\hat{x},\hat{c}))}{\mathbb{S}^2}\leq \frac{\bar{\rho}E_S(Var_{X,Y}(\mathbb{R}(S)))}{\mathbb{S}^2}\leq \frac{\bar{\rho}(1-\mathbb{S}^2)}{\mathbb{S}^2}.
\end{equation}
The theorem is proved.
\end{proof}

\subsection{Time Complexity}
The running time of the COST algorithm is mainly consumed by three procedures: subspace generation and selection, \textit{p}-value calculation and the selection of parameter \textit{r}.
\subsubsection{Subspace generation and selection}
In this step, we need to iterate $b_1$ times, with each iteration generating $b_2$ candidate subspaces and selecting the best one, ultimately resulting in $b_1$ subspaces being produced.

In each iteration, we randomly generate $b_2$ feature subsets of $F=(a_1,...,a_d)$. The time complexity required to obtain the frequency distributions of all distinct feature value sets on every class for each feature subset is at most $O((n+m)\cdot d)$ with the help of a hash table. Therefore, the time complexity required for all $b_2$ candidate subspaces is $O(b_2\cdot d\cdot (n+m))$. Subsequently, we aim to select the candidate subspace with the highest relative risk.

For a given candidate subspace $cs_i,1\leq i \leq b_2$, we assume that it has $u$ unique feature value sets. We can collect the frequency of each unique feature value set in every class in $O(n)$ time. Consequently, we can compute the relative risk of each feature value in constant time $O(1)$ using Equation (\ref{math:single-RR}). Therefore, the time complexity required to calculate the sum of relative risks of all $u$ feature value sets and to obtain the average relative risk of $cs_i$ is $O(u)$. In summary, the time complexity for calculating the relative risks of $b_2$ candidate subspaces and selecting the best one is $O(b_2\cdot(n+u))$. In the worst-case scenario, when $u$ equals $n$, the time complexity is $O(b_2\cdot(n+n))=O(b_2\cdot n)$.

After iterating $b_1$ times, we have a time complexity of $O(b_1\cdot b_2 \cdot d\cdot(n+m)+b_1\cdot b_2 \cdot n)=O(b_1\cdot b_2 \cdot d\cdot (n+m))$.

\subsubsection{\textit{P}-value calculation}
For each subspace $s_{i}$, suppose it has $u$ distinct feature value sets. To obtain the frequency of each feature value set within each class, we can just scan the training set once to fulfill this task. If the feature value set and each class label are stored in a hash table, we can finish this task in a linear expected time $O(n)$.  

Based on the frequency information collected above, we can swiftly construct contingency tables akin to Table \ref{table:fisher} for each class. If we compute the exact \textit{p}-value according to Equation (\ref{math:p-value}), then the worst-case time complexity will be $O(n)$. To improve the running efficiency, we can adopt the \textit{p}-value upper bound presented by \cite{hamalainen2016new} instead of the exact \textit{p}-value. This upper bound can be calculated in $O(1)$ if factorials up to $n!$ have been calculated in advance and stored in the main memory. Hence, the time complexity of calculating \textit{p}-values for all unique feature value sets across all classes is $O(uk)$, where $k$ is the number of classes. Consequently, the overall time complexity of calculating all \textit{p}-values on the subspace $s_{i}$ is  $O(n+uk)$. Note that in the worst-case scenario, where $u$ equals $n$, hence the worst-case time complexity is $O(n+nk)=O(nk)$. Considering all $b_1$ subspaces, the overall time complexity of this step is $O(b_1 \cdot n \cdot k)$.
\subsubsection{The selection of parameter “\textit{r}”}
In the previous step, we have computed \textit{p}-values for all possible feature value sets with respect to each class. Consequently, for the validation set, we can construct a $m\times b_1$ matrix for each class, where $m$ is the number of validation samples and $b_1$ is the number of subspaces. The element in the $i$-th row and $j$-th column of this matrix represents the \textit{p}-value for the $i$-th validation sample in the $j$-th subspace with respect to that class. The time complexity for constructing such $k$ tables is $O(k\cdot b_1\cdot m)$. Subsequently, for each class-specific table, sorting \textit{p}-values from $b_1$ subspaces for each sample in each row requires a time complexity of $O(b_1\cdot logb_1)$. Therefore, the time complexity for sorting \textit{p}-values in all  $k$ tables is $O(k\cdot m\cdot b_1 \cdot logb_1)$.

Next, for each sorted table corresponding to each class, we select the $m$ \textit{p}-values in the  \textit{r}-th column to obtain the final \textit{p}-values for all $m$ validation samples with respect to each class and the time complexity is $O(m\cdot k)$. As a result, we obtain a table of size $m\times k$, where the element in the \textit{i}-th row and \textit{j}-th column represents the \textit{p}-value for the \textit{i}-th validation sample in the \textit{j}-th class. 

Finally, we assign the class label with the smallest \textit{p}-value to each sample, thus obtaining the predicted labels for the validation set under the given parameter $r$, and the time complexity is $O(m\cdot k)$. Comparing these predictions with the ground-truth labels yields the classification accuracy for the validation set under the parameter $r$, with the time complexity of this step being $O(m)$. By iterating over every possible $r$, we select the $r$ with the highest accuracy as the optimal parameter. 

Therefore, the overall time complexity of this part is $O(k\cdot b_1 \cdot m+k\cdot m \cdot b_1\cdot logb_1)+O(b_1\cdot(m\cdot k+m\cdot k+m))=O(k\cdot m \cdot b_1\cdot logb_1)$.

\subsubsection{Overall time complexity}
In summary, the overall time complexity of COST is the sum of time complexities of above three parts: $O(b_1\cdot b_2\cdot d\cdot (n+m)+b_1\cdot n\cdot k+k\cdot m\cdot b_1 \cdot logb_1)$.

\section{Experiments}\label{experiments}
We conduct comprehensive experiments on 28 real-world data sets to assess the performance of the COST method. In particular, our objective is to address the following Research Questions (RQs):
\begin{itemize}
	\item RQ1: Can COST achieve comparable performance to classic classifiers such as random forest and those state-of-the-art random subspace ensemble classification algorithms in regular classification?
	\item RQ2: Can COST achieve better performance than existing classifiers in the context of selective classification?
	\item RQ3: Is COST robust to its parameters and the \textit{p}-value combination method?
\end{itemize}
\subsection{Experimental Setup}
\subsubsection{Data sets}
The data sets used in the experiments could be downloaded from \cite{dua2017uci} and \cite{derrac2015keel}. The main characteristics of $28$ datasets are provided in Table \ref{table:dataset}. We treat the missing value in each feature as a special feature value.

\begin{table}[h]
	\centering
	\caption{The main characteristics of 28 data sets.}
	\begin{tabular}{cccc}
		\toprule
		\label{table:dataset}
		Dataset & \#Samples & \#Features &\#Classes \\ 
		\midrule
		Balance Scale & 625& 4&3 \\ 
		Breast Cancer & 699 & 9&2 \\ 
		Car & 1728& 6&4  \\ 
		Chess & 3196&36 &2 \\
		Cleveland  &303 &13 &5 \\ 
		Dermatology  & 366& 34&6 \\ 
		Dna-promoter  &106 & 57&2 \\ 
		Haberman  &  306& 3&2 \\ 
		Happiness & 143&6 &2 \\ 
		Hayes-Roth & 132&4 &3  \\ 
		Heart &270 &13 & 2 \\ 
		House-votes  &435 &16 &2 \\ 
		Iris &150 & 4& 3\\ 
		Led7digit & 500&7 &10 \\ 
		Lymphography &148 &18 &4 \\ 
		Monks-2 & 432& 6&2 \\ 
		Mushroom & 8124&22 &2 \\ 
		Newthyroid & 215& 5& 3\\ 
		Nursery &11960 &8 &5 \\ 
		Pima &768 &8 &2 \\ 
		Seeds & 210&7 &3 \\ 
		Solar-flare & 323& 9& 6\\ 
		Soybean-Small &47 &35 &4 \\ 
		Tic-tac-toe &958 &9 &2 \\ 
		Titanic & 2201& 3& 2\\ 
		Vehicle & 846& 18&4 \\ 
		Wine &178 &13 &3 \\ 
		Zoo & 101& 16&7 \\ 
		\bottomrule
	\end{tabular}
\end{table}

\subsubsection{Evaluation measures}
\label{section:evaluation}
\begin{itemize}
	\item In the context of regular classification, the standard classification accuracy is employed as the evaluation criterion. In the experiment, we repeat a five-fold cross-validation 10 times to obtain the average accuracy on each data set.
	\item In selective classification, a universally accepted evaluation criterion is still absent. For each test sample $\hat{x}$, let $\mathsf{y}$ denotes the set of its true labels and $\mathsf{\hat{y}}$ denotes the set of labels predicted by a classifier with refine and reject options. Note that if the test sample is an outlier or it is rejected, an additional special class label $c_0$ is introduced in such scenarios. Then, we can use the Jaccard coefficient to evaluate the prediction result $\mathsf{\hat{y}}$ of a selective classifier on the test sample $\hat{x}$: 
	\begin{equation}
		JC(\hat{x}) = \frac{|\mathsf{\hat{y}} \cap \mathsf{y}|}{| \mathsf{\hat{y}}\cup \mathsf{y}|}.
	\end{equation}
	 The average value of the above coefficients for all test samples, is named as Jaccard accuracy (denoted by JacAcc), can be used as the evaluation metric for comparing different selective classifiers.
	 
	 We have the following remarks for this new evaluation metric:
	 \begin{enumerate}
	 	\item In the context of regular classification, JacAcc will be identical to the standard classification accuracy since $|\mathsf{y}|=|\mathsf{\hat{y}}|=1$ and JacAcc = 1 or 0. 
	 	\item In the context of classification with reject option, we have several possibilities. First, if the test sample is not an outlier ($c_0 \notin \mathsf{y}$) and should not be rejected, we will make an error if the reject option is adopted so that $JC(\hat{x})=0$ in this case. Similarly, $JC(\hat{x})$ will be $0$ as well if the test sample $\hat{x}$ is an outlier and it is not rejected. If $\hat{x}$ is an outlier and we reject it, then $JC(\hat{x})=1$.
	 	\item  In the context of classification with refine option, $JC(\hat{x})$ will equal to 1 only $\mathsf{y}$ and $\mathsf{\hat{y}}$ are identical. Meanwhile, $JC(\hat{x})$ will be reduced if we include more irrelevant labels in $\mathsf{\hat{y}}$. 
	 \end{enumerate}
\end{itemize}

\subsubsection{Baseline methods}
In the performance comparison, we include the following baseline methods:
\begin{itemize}
	\item The \textbf{classic classifiers}: Naive Bayes (NB), Support Vector Machine (SVM), Decision Tree (DT), and Random Forest (RF). We use implementations in the scikit-learn package \cite{pedregosa2011scikit} with their default parameter settings (the default number of trees of RF is 100). The categorical features are first transformed into continuous ones through one-hot encoding before being fed into NB and SVM.
	\item The \textbf{state-of-the-art random subspace ensemble classification algorithms}: \textbf{Ra}ndom \textbf{S}ubsapce \textbf{E}nsemble (RaSE, \cite{tian2021rase}) and \textbf{P}arametric \textbf{R}andom \textbf{S}ubspace (PRS, \cite{huynh2023optimizing}). The source codes for RaSE and PRS are available online at \url{https://cran.r-project.org/web/packages/RaSEn/} and \url{https://github.com/vahuynh/PRS/tree/main/code}, respectively. As the example code in PRS utilizes k-nearest neighbors (kNN) as the base classifier, we have also chosen kNN as the base classifier for both PRS and RaSE. Similar to RF, the number of base classifiers for RaSE and PRS is set to be 100. For each base classifier, the number of candidate subspaces is fixed to be 10, and all other parameters are specified to be their default values. Similar to NB and SVM, the categorical features are first transformed into continuous features through one-hot encoding before being fed into RaSE and PRS.
	\item The\textbf{ conformal ensemble predictors}: Fisher method \cite{toccaceli2017combination} and Majority-vote method \cite{cherubin2019majority}. The three base classifiers for both methods are DT, SVM and kNN. The source code for conformal prediction is available in the Python package “crepes” (\url{https://github.com/henrikbostrom/crepes}). In the Fisher method, we select the class with the highest \textit{p}-value as the predicted class; In the Majority-vote method, we choose the class with the most votes as the predicted class.
	\item The \textbf{selective classifiers} with both reject and refine option: \textbf{BCOPS} \cite{guan2022prediction} and \textbf{GPS} \cite{zhou2023JMLR}. The source codes for BCOPS and GPS are available online at \url{https://github.com/LeyingGuan/BCOPS} and \url{https://github.com/Zhou198/GPS}, respectively. Both BCOPS and GPS are executed under their default parameter values.
\end{itemize}
 COST is an algorithm specifically designed for categorical data, thus we discretize continuous features via $k$-means in which the number of clusters (i.e., the number of feature values) is set to be the number of classes. Similarly, the number of chosen subspaces for COST is set to be 100, and the number of candidate subspaces in each round is set to be 10. In addition to the randomly chosen subspaces, COST includes the original features as well. That is, $100+d$ subspaces are used by COST in which each of $d$ subspaces corresponds to one of $d$ features in the data set. The source code of COST is available at \url{https://github.com/zrli2000/COST}.

\subsection{Investigation of RQ1}
For the task of regular classification, the performance comparison results in terms of classification accuracy, are depicted in Table \ref{table:acc}. From this table, some important observations are summarized as follows. 

Firstly, compared to those classic classifiers, COST is slightly better than NB, DT and SVM. This happens probably because COST is an ensemble classification method so that it can yield better performance than these non-ensemble classifiers. However, the classification accuracy of COST is worse than that of RF mainly due to the fact the base classifier in COST is much simpler than decision tree in RF. Overall, COST is comparable to these classic classifiers in terms of the classification accuracy. 

Secondly, COST has quite similar performance to RaSE and PRS, two recently proposed subspace ensemble classification methods. This can be attributed to the fact that COST follows the same methodology used in RaSE for candidate subspace selection. Meanwhile, our method can also achieve comparable performance to two conformal ensemble classifiers.

\begin{table}[!htbp]
	\centering
	\caption{The performance comparison of different classifiers in the term of classification accuracy. On several data sets, the results of RaSE and PRS are N/A because the number of samples in certain classes is less than the minimal class size threshold that is required by these two algorithms. For the same reason, the results of CP-Fisher and CP-Vote on some data sets are N/A as well.}
	\begin{tabular}{cccccccccc}
		\toprule
		\label{table:acc}
		Dataset          & COST  & RaSE  & PRS & CP-Fisher & CP-Vote  & RF    & NB    & SVM   & DT    \\
		\midrule
		Balance Scale    & 0.84 & 0.84 & 0.86 & 0.85       & 0.88     & 0.84 & 0.90 & 0.91 & 0.78 \\
		Breast Cancer & 0.96 & 0.97 & 0.97 & 0.96       & 0.95     & 0.97 & 0.96 & 0.97 & 0.94 \\
		Car              & 0.95 & 0.87 & 0.92 & 0.96       & 0.94     & 0.98 & 0.80 & 0.97 & 0.98 \\
		Chess            & 0.93 & 0.97 & 0.98 & 0.98       & 0.97     & 0.99 & 0.63 & 0.97 & 1.00 \\
		Cleveland        & 0.59 & 0.56 & 0.55 & 0.54       & 0.54     & 0.57 & 0.54 & 0.54 & 0.49 \\
		Dermatology      & 0.96 & 0.97 & 0.97 & 0.90       & 0.88     & 0.98 & 0.88 & 0.71 & 0.94 \\
		Dna-promoter     & 0.79 & 0.81 & 0.89 & 0.86       & 0.89     & 0.88 & 0.87 & 0.92 & 0.72 \\
		Haberman         & 0.76 & 0.73 & 0.75 & 0.74       & 0.74     & 0.69 & 0.74 & 0.73 & 0.66 \\
		Happiness        & 0.63 & 0.57 & 0.53 & 0.59       & 0.53     & 0.59 & 0.58 & 0.56 & 0.52 \\
		Hayes-Roth            & 0.77 & 0.72 & 0.77 & 0.71       & 0.70     & 0.82 & 0.78 & 0.78 & 0.82 \\
		Heart            & 0.82 & 0.83 & 0.82 & 0.71       & 0.66     & 0.83 & 0.81 & 0.68 & 0.73 \\
		House-votes      & 0.95 & 0.95 & 0.96 & 0.95       & 0.94     & 0.96 & 0.94 & 0.95 & 0.94 \\
		Iris             & 0.95 & 0.95 & 0.96 & 0.95       & 0.95     & 0.95 & 0.95 & 0.97 & 0.94 \\
		Led7digit        & 0.73 & 0.68 & 0.63 & 0.70       & 0.55     & 0.71 & 0.62 & 0.72 & 0.70 \\
		Lymphography     & 0.81 &   N/A   &  N/A    &    N/A        &    N/A      & 0.84 & 0.72 & 0.83 & 0.78 \\
		Monks-2            & 0.98 & 0.99 & 1.00 & 0.99       & 0.99     & 0.99 & 0.80 & 0.99 & 1.00 \\
		Mushroom         & 1.00 & 1.00 & 1.00 & 1.00       & 1.00     & 1.00 & 0.96 & 1.00 & 1.00 \\
		Newthyroid       & 0.91 & 0.95 & 0.93 & 0.89       & 0.91     & 0.96 & 0.96 & 0.83 & 0.93 \\
		Nursery          & 0.97 & N/A     &   N/A   &     N/A       &    N/A      & 0.98 & 0.84 & 1.00 & 1.00 \\
		Pima             & 0.74 & 0.76 & 0.75 & 0.75       & 0.73     & 0.76 & 0.76 & 0.76 & 0.69 \\
		Seeds            & 0.89 & 0.94 & 0.91 & 0.89       & 0.88     & 0.93 & 0.90 & 0.90 & 0.91 \\
		Solar-flare      & 0.69 & 0.71 & 0.72 & 0.70       & 0.70     & 0.72 & 0.57 & 0.73 & 0.70 \\
		Soybean-Small          & 0.95 &  N/A    &  N/A    & 1.00       & 1.00     & 1.00 & 1.00 & 1.00 & 0.96 \\
		Tic-tac-toe      & 0.92 & 0.94 & 0.92 & 0.95       & 0.95     & 0.95 & 0.67 & 0.99 & 0.89 \\
		Titanic          & 0.78 & 0.77 & 0.72 & 0.79       & 0.78     & 0.79 & 0.77 & 0.78 & 0.79 \\
		Vehicle          & 0.66 & 0.72 & 0.72 & 0.64       & 0.56     & 0.75 & 0.46 & 0.49 & 0.70 \\
		Wine             & 0.92 & 0.97 & 0.96 & 0.88       & 0.80     & 0.98 & 0.97 & 0.68 & 0.91 \\
		Zoo              & 0.92 & 0.89 & 0.91 & 0.88       & 0.90     & 0.96 & 0.95 & 0.94 & 0.95 \\
		\midrule
		Average Acc         & 0.85 & 0.84 & 0.84 & 0.84       & 0.82     & 0.87 & 0.80 & 0.83 & 0.84\\
		Average Rank & 5.13 & 4.80 & 4.18 & 5.52 & 6.27& 2.75&5.82 & 4.34 & 5.45 \\
		\bottomrule
	\end{tabular}
\end{table}

We also compared the running time of these classification algorithms, as presented in Figure \ref{figure2}. It indicates that the COST algorithm demonstrates a notable efficiency advantage over other subspace ensemble methods such as RaSE and PRS. However, it's important to note that while COST excels within its sub-category, it still lags behind traditional, non-ensemble algorithms like DT, SVM, and NB. This suggests that we need to further improve its running efficiency in order to handle large data sets in practice. 

\begin{figure}[!htbp]
	\centering
	\includegraphics[width=\columnwidth]{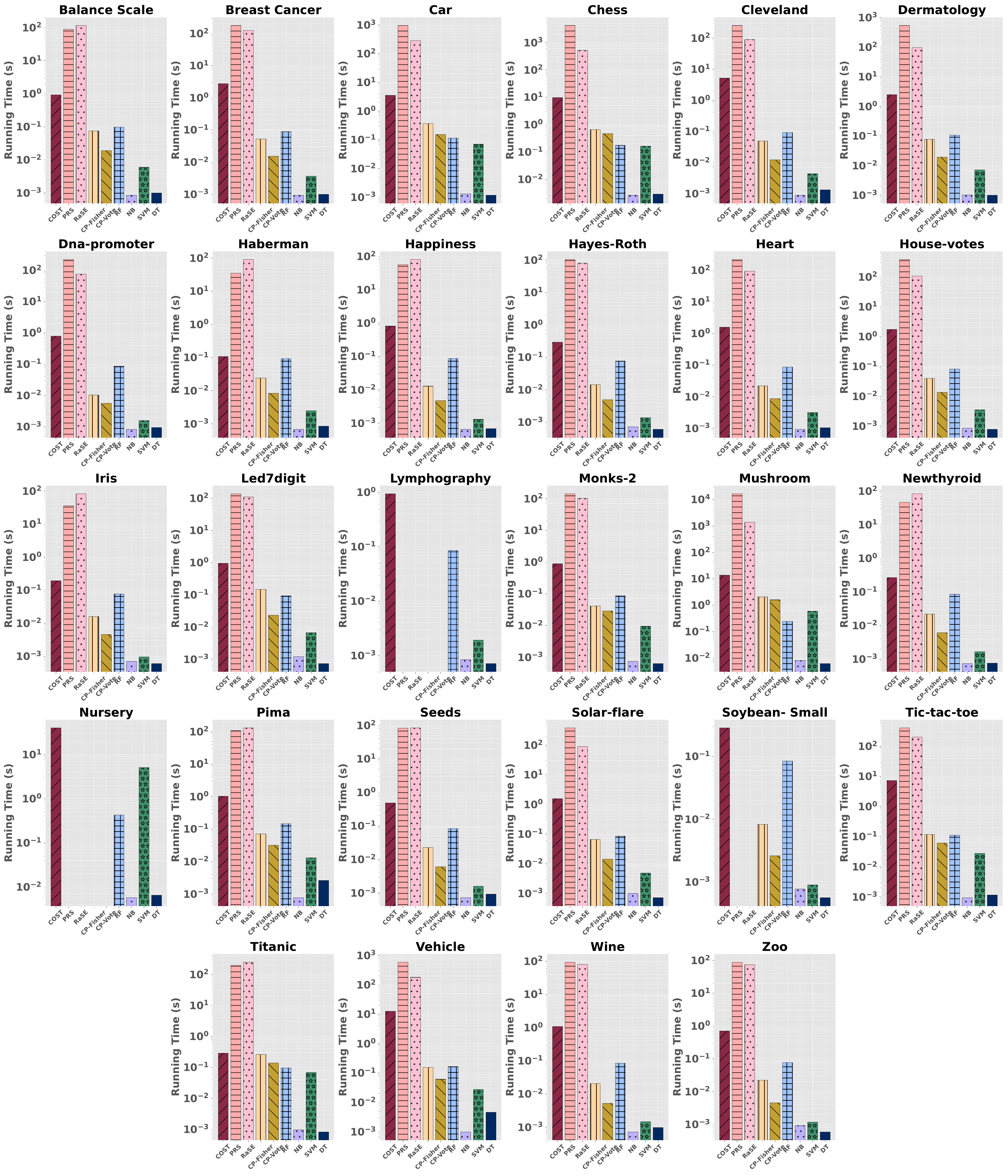}
	\caption{Running time comparison of different classifiers. The running time was measured in seconds and all experiments were conducted on a PC equipped with an M1 CPU and 16GB of memory.}
	\label{figure2}
\end{figure}

\subsection{Investigation of RQ2}

To test the performance of COST in the context of selective classification, we compare it with BCOPS and GPS by employing the JacAcc defined in Section \ref{section:evaluation} as the performance indicator. The significance level parameter $\alpha$ for COST, BCOPS, and GPS is set to be 0.05. Moreover, in line with the recommendations from the authors of the GPS algorithm, we have used the test data as “unlabeled data”, combining it with the training data to form a new training set for GPS.  Correspondingly, the calibration data is also treated as “unlabeled data” and combined with the original calibration data to constitute a new calibration set for GPS.

The detailed experimental results in terms of JacACC are given in Table \ref{table:normal}. Note that all samples in these 28 data sets are associated with only one label and it is assumed that no outliers are present. Hence, one algorithm will achieve better performance in terms of JacAcc if it returns smaller prediction sets and opts for fewer reject options on these data sets. Excessively large prediction sets, even if they include the correct class labels, will result in lower JacACC scores. Consequently, the ideal prediction scenario for these data sets would involve a prediction set containing only one correct class label. 

\begin{table}[h]
	\centering
	\caption{The performance comparison of COST, GPS and BCOPS in the term of JacAcc. On several data sets, the results of GPS are N/A because it reports errors due to unknown reasons.}
	\setlength{\tabcolsep}{5mm}
	\begin{tabular}{cccc}
		\toprule
		\label{table:normal}
		Dataset          & COST  & GPS  & BCOPS      \\
		\midrule
		Balance Scale    & 0.64 & 0.48 & 0.44 \\
		Breast Cancer & 0.92 & 0.89 & 0.94 \\
		Car              & 0.26 & 0.76 & 0.78 \\
		Chess            & 0.88 & 0.94 & 0.92 \\
		Cleveland        & 0.47 & 0.21 & 0.21 \\
		Dermatology      & 0.59 & 0.66 & 0.38 \\
		Dna-promoter     & 0.27 & 0.50  & 0.56 \\
		Haberman         & 0.75 & 0.50  & 0.51 \\
		Happiness        & 0.10  & 0.50  & 0.50  \\
		Hayes-Roth            & 0.38 & 0.43 & 0.36 \\
		Heart            & 0.68 & 0.50  & 0.03 \\
		House-votes      & 0.68 & 0.82 & 0.89 \\
		Iris             & 0.86 & 0.90  & 0.60  \\
		Led7digit        & 0.43 &   N/A   & 0.17 \\
		Lymphography     & 0.61 &  N/A    & 0.28 \\
		Monks-2            & 0.85 & 0.94 & 0.92 \\
		Mushroom         & 1.00    & 0.98 & 0.95 \\
		Newthyroid       & 0.82 & 0.86 & 0.36 \\
		Nursery          & 0.83 &  N/A    & 0.63 \\
		Pima             & 0.62 & 0.50  & 0.64 \\
		Seeds            & 0.83 & 0.77 & 0.82 \\
		Solar-flare      & 0.62 & 0.39 & 0.26 \\
		Soybean-Small          & 0.65 & 0.67 & 0.25 \\
		Tic-tac-toe      & 0.12 & 0.50  & 0.63 \\
		Titanic          & 0.73 &   N/A   & 0.57 \\
		Vehicle          & 0.44 & 0.43 & 0.58 \\
		Wine             & 0.75 & 0.58 & 0.73 \\
		Zoo              & 0.77 & 0.81 & 0.15\\
		\midrule
		Average JacAcc        & 0.63 & 0.65 & 0.54  \\
		Average Rank     & 1.82 & 1.92 & 2.11 \\
		\bottomrule
	\end{tabular}
\end{table}

The results depicted in Table \ref{table:normal} suggest that our proposed algorithm COST outperforms BCOPS and GPS on the majority of the datasets in terms of JacAcc scores. To further reveal the reason behind this observation, we plot the refine rate (the proportion of test samples with more than one predicted label) and reject rate (the proportion of rejected test samples) of each algorithm in Figure \ref{figure4} and Figure \ref{figure5}, respectively. The results from Figure \ref{figure4} suggest that BCOPS and GPS exhibit a strong propensity to employ the refine option on the majority of datasets, whereas COST shows a tendency to choose the refine option on only a few datasets. Conversely, the findings from Figure \ref{figure5} reveal that while BCOPS and GPS almost never opt for the reject option, COST demonstrates a marked preference for the reject option on several datasets. Overall, COST demonstrates superior performance compared to the other two selective classification algorithms.

\begin{figure}[!htbp]
	\centering
	\includegraphics[width=\columnwidth]{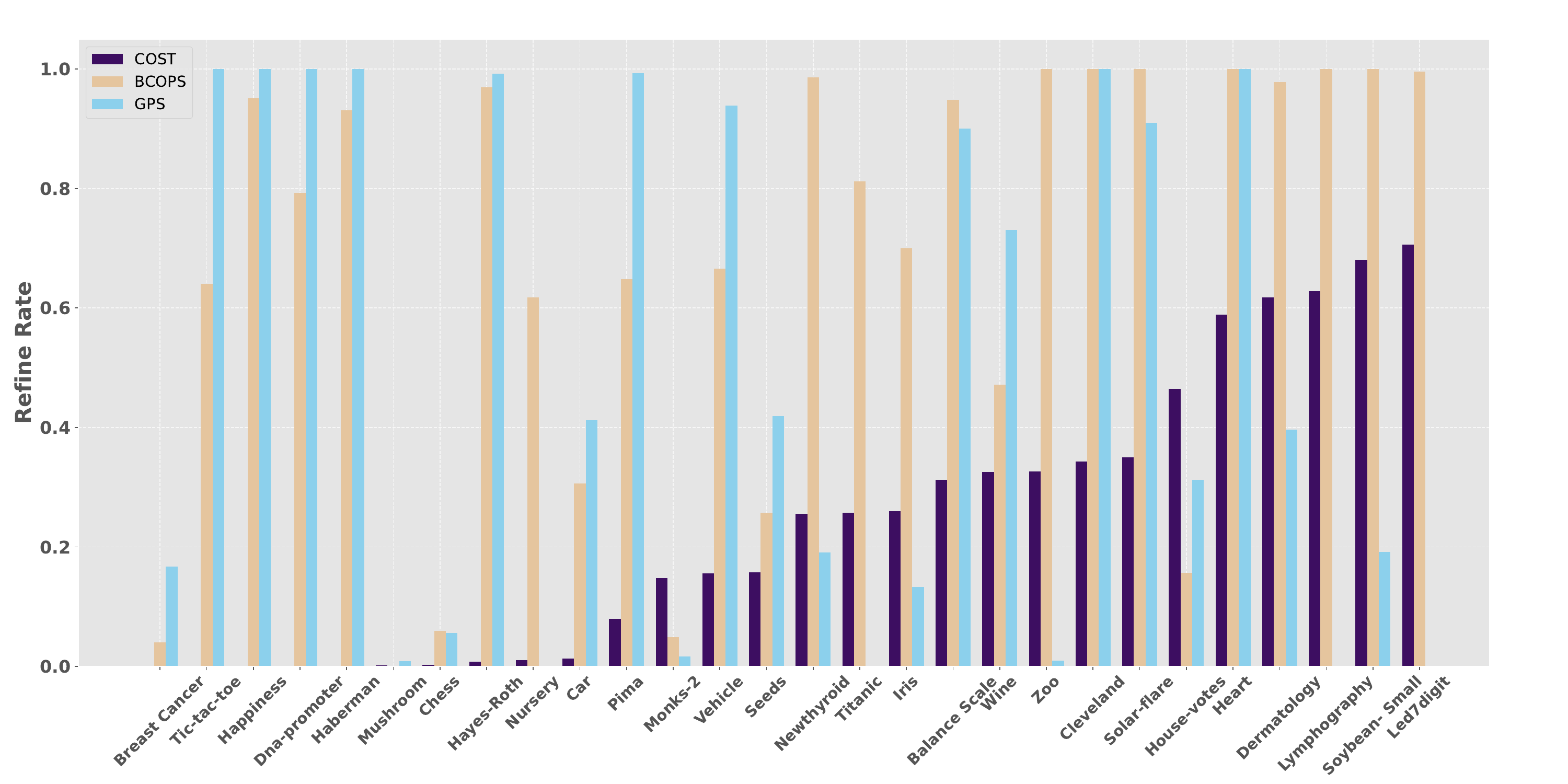}
	\caption{The comparison of COST, BCOPS, and GPS with respect to the refine rate.}
	\label{figure4}
\end{figure}

\begin{figure}[!htbp]
	\centering
	\includegraphics[width=\columnwidth]{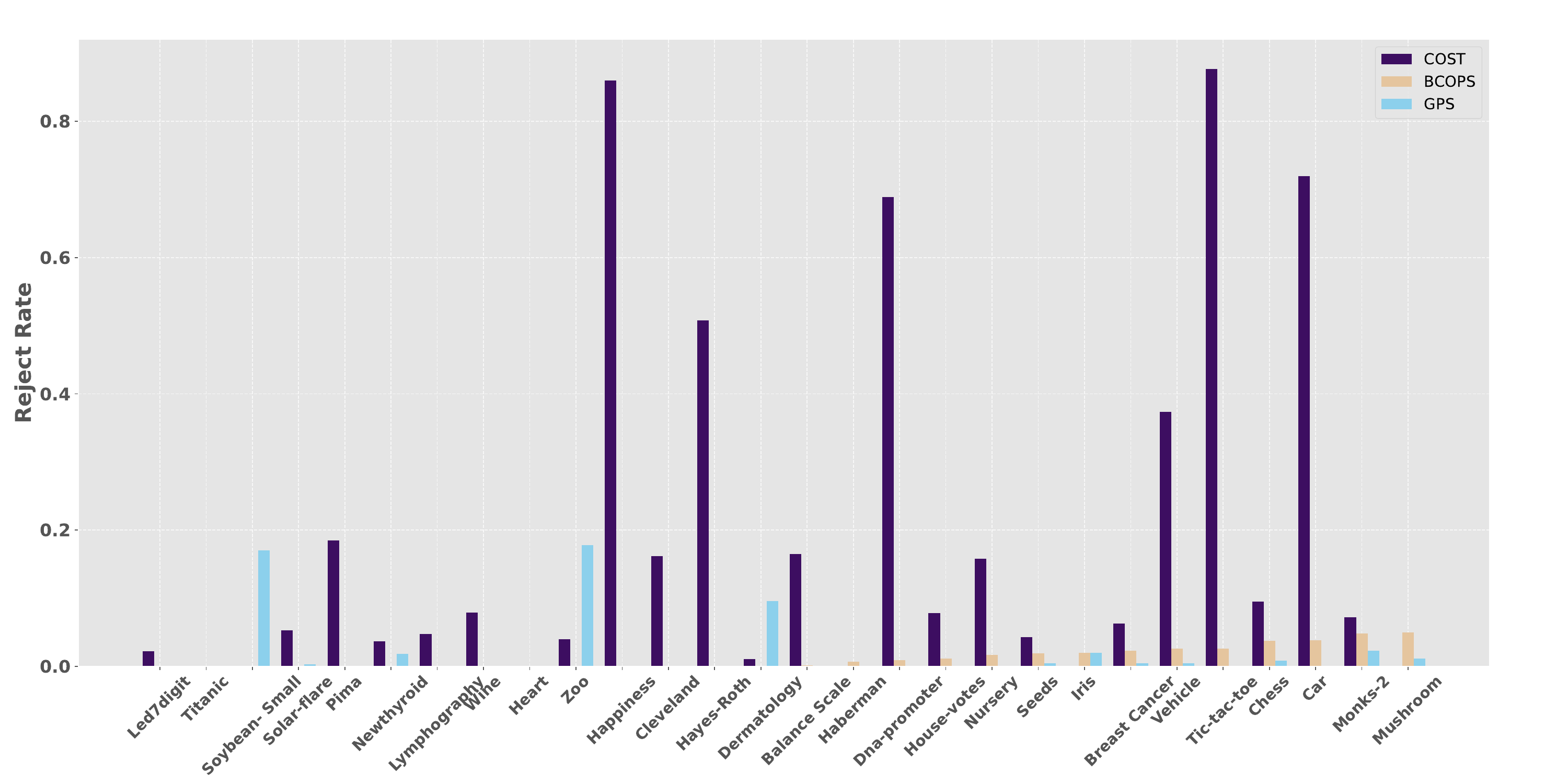}
	\caption{The comparison of COST, BCOPS, and GPS with respect to the reject rate.}
	\label{figure5}
\end{figure}

Furthermore, we also conducted experiments on datasets with outliers to assess the capabilities of these algorithms on identifying outliers by using the reject option. Similar to \cite{tax2008growing}, we selected four datasets with more than two classes and designated samples in one class as outliers. The outlying samples were not included in the training set but were employed during the testing stage. Non-outlier samples were equally divided into two sets: one for training and the other, combined with outliers, for testing. Detailed information regarding these four datasets is provided in Table \ref{table:outlier-datasets}.

\begin{table}[!htbp]
	\centering
	\caption{The data sets in which samples from one class are regarded as outliers. Numbers with a superscript asterisk $^*$ denote the sample size of the outlier class.}
	\setlength{\tabcolsep}{1.1mm}
	\begin{tabular}{cccccc}
		\hline
			\label{table:outlier-datasets}
		Dataset       & \begin{tabular}[c]{@{}c@{}}Normal\\ classes\end{tabular} & \begin{tabular}[c]{@{}c@{}}Outlier\\ class\end{tabular} & \begin{tabular}[c]{@{}c@{}}\#Training\\ samples\end{tabular} & \begin{tabular}[c]{@{}c@{}}\#Testing\\ samples\end{tabular} & Outlier Ratio \\ \hline
		Hayes         & 1-2      & 3      & 51     & $51+30^*=81$     & 22.73\%     \\
		Solar-flare   & 1-5      & 6      & 129    & $129+65^*=194$   & 20.12\%     \\
		Soybean-Small & 1-3      & 4      & 15     & $15+17^*=32$     & 36.17\%     \\
		Zoo           & 1-4,6-7  & 5      & 49     & $48+4^*=52$      & 3.96\%      \\ \hline
	\end{tabular}
\end{table}

In this experiment, to ensure that all algorithms were unaware of the presence of outliers, for the GPS algorithm, we randomly selected only two outlying samples as “unlabeled data” for its training process. Similarly, for all algorithms, the significance level parameter $\alpha$ was set to be 0.05. The detailed experimental results on these data sets with outliers are recorded in Table \ref{table:outlier-result}.

\begin{table}[!htbp]
	\centering
	\caption{The experimental results on data sets in which samples in one class are regarded as outliers and excluded from the training data. For each algorithm on each data set, we record both the JacAcc score and the proportion of outliers that were successfully identified.}
	\setlength{\tabcolsep}{6mm}
	\begin{tabular}{cccc}
		\toprule
		\label{table:outlier-result}
		Dataset & COST & BCOPS  & GPS \\
		\midrule
		Hayes&0.40 (100\%)&0.31 (0\%)&0.37 (14.28\%)\\
%		Led7digit&&&\\
		Solar-flare&0.63 (76.92\%)& 0.14 (0\%)&0.40 (66.67\%)\\
		Soybean-Small&0.53 (100\%)& 0.16 (0\%)&0.80 (66.67\%)\\
		Zoo&0.81 (0\%)& 0.15 (0\%)& 0.47 (50.00\%)\\
		\midrule
		Average & 0.59 (69.23\%) & 0.19 (0\%)&0.51 (49.40\%)  \\
		\bottomrule
	\end{tabular}
\end{table}

The results in Table \ref{table:outlier-result} indicate that, on average, the COST algorithm outperforms both GPS and BCOPS, in terms of both JacAcc score and outlier detection rate.

To assess the capabilities of these algorithms on classifying samples with multiple labels by using the refine option, we conducted experiments on four multi-label datasets from Mulan library \cite{tsoumakas2011mulan}. The detailed information regarding four multi-label datasets is provided in Table \ref{table:multi-label-datasets}.

\begin{table}[!htbp]
	\centering
	\caption{The main characteristics of 4 multi-label data sets. The average label cardinality reflects the average number of labels associated with each sample across the entire dataset.}
	\setlength{\tabcolsep}{3.2mm}
	\begin{tabular}{ccccc}
		\toprule
		\label{table:multi-label-datasets}
		Dataset & \#Samples & \#Features & \#Classes &\#Average Label Cardinality \\
		\midrule
		Emotions&593&72 &6 &1.87\\
		Flags&194&19&7&3.39\\
		Scene&2407&294&6 &1.07\\
		Yeast&2417&103&14 &4.24\\
		\bottomrule
	\end{tabular}
\end{table}

In this experiment, we randomly partitioned each data set into training and testing sets, with 70\% samples in the training set and 30\% samples in the testing set. For all algorithms, the significance level parameter $\alpha$ was set to be 0.05. For the GPS and BCOPS algorithms, we converted the multi-label samples in the training set into single-label samples, while the test samples remained unchanged because these two algorithms are incapable of accepting multi-label training data. The detailed experimental results on these data sets are recorded in Tabel \ref{table:multi-label-result}. 

\begin{table}[!htbp]
	\centering
	\caption{The experimental results on 4 multi-label data sets.}
	\setlength{\tabcolsep}{1.4mm} % Adjust the space between columns as needed to fit the new columns
	\begin{tabular}{cccccccccc}
		\toprule
		\label{table:multi-label-result}
		 & \multicolumn{3}{c}{JacAcc} & \multicolumn{3}{c}{\parbox{4cm}{\centering Average number of predicted labels}} & \multicolumn{3}{c}{\parbox{4cm}{\centering Average number of correctly predicted ground-truth labels}} \\
		\cmidrule(r){2-4} \cmidrule(r){5-7} \cmidrule(r){8-10}
		Dataset & COST & BCOPS & GPS & COST & BCOPS & GPS & COST & BCOPS & GPS \\
		\midrule
		Emotions & 0.53 & 0.43 & 0.32 & 2.21 & 4.42 & 5.98 & 1.31 &  1.84& 1.89 \\
		Flags    & 0.30 & 0.47 & 0.49 & 2.00 & 6.79 & 6.93 & 1.27 & 3.28 & 3.44 \\
		Scene    & 0.50 & 0.51 & 0.48 & 2.04 & 2.35 & 1.89 & 0.89 & 1.01 & 0.85 \\
		Yeast    & 0.30 & 0.33 & 0.36 & 4.11 & 12.39 & 12.23 & 1.67 & 4.04 & 4.30 \\
		\midrule
		Average  & 0.41 & 0.44 & 0.41 & 2.59 & 6.49 & 6.76 & 1.23 & 2.54 & 2.62 \\
		\bottomrule
	\end{tabular}
\end{table}

From Table \ref{table:multi-label-result}, it is evident that the performance of COST is not optimal in terms of the JacAcc scores. However, the average number of predicted labels of COST is quite close to the average label cardinality in Table \ref{table:multi-label-datasets} for each data set. In contrast, both BCOPS and GPS predict an average number of labels that is almost close to the number of all class labels except for the Scene data set. This suggests that both algorithms are more inclined to adopt the refine option. Compared to the other two algorithms, COST is more inclined towards adopting the reject option and hence the Type-I error in its prediction set is lower.

\subsection{Investigation of RQ3}

Several parameters are required as input in COST. The core parameters of COST include $b_1$ and $b_2$. Therefore, we focus on how the parameters $b_1$ and $b_2$ influence the classification performance and running time of the COST algorithm.

We vary $b_1$ from 10 to 100 and record the JacAcc scores and the running time of COST on different datasets in Figure \ref{figure:b1JA} and Figure \ref{figure:b1Time}. From Figure \ref{figure:b1JA}, it can be observed that as $b_1$ increases, the corresponding JacAcc score exhibits stability, with minimal fluctuations observed. This indicates that, from the perspective of JacAcc scores, COST is insensitive to the value of $b_1$. From Figure \ref{figure:b1Time}, as expected, we can see that as $b_1$ increases, the corresponding runtime of the COST algorithm also increases.

Figure \ref{figure:b2JA} and Figure \ref{figure:b2Time} present the JacAcc scores and the running time of COST when $b_2$ is varied from 1 to 15. From Figure \ref{figure:b2JA}, it can be observed that as $b_2$ increases, the corresponding JacAcc score also exhibits stability. It indicates that COST is also insensitive to the value of $b_2$ in terms of the JacAcc score. From Figure \ref{figure:b2Time}, we can see that similar to $b_1$, the runtime of the COST algorithm also increases as the value of $b_2$ increases.

Finally, to evaluate the impact of different \textit{p}-value combination methods, we selected four distinct methods: rOP, Fisher, minP, and maxP. Detailed performance comparison results are illustrated in Figure \ref{figure:combination-JA}. From Figure \ref{figure:combination-JA}, we can see that the performance of maxP is worst. This is because the maxP method selects the largest \textit{p}-value as the test statistic. As a result, the COST algorithm will have a significant bias towards adopting the reject option. In contrast, the minP method opts for the smallest \textit{p}-value as the test statistic, which results in a tendency to favor the refine option. While after the training and validation process, rOP method selects a “\textit{r}” value that yields the best performance, thereby achieving better classification performance.

\begin{figure}[!htbp]
	\centering
	\includegraphics[width=\columnwidth]{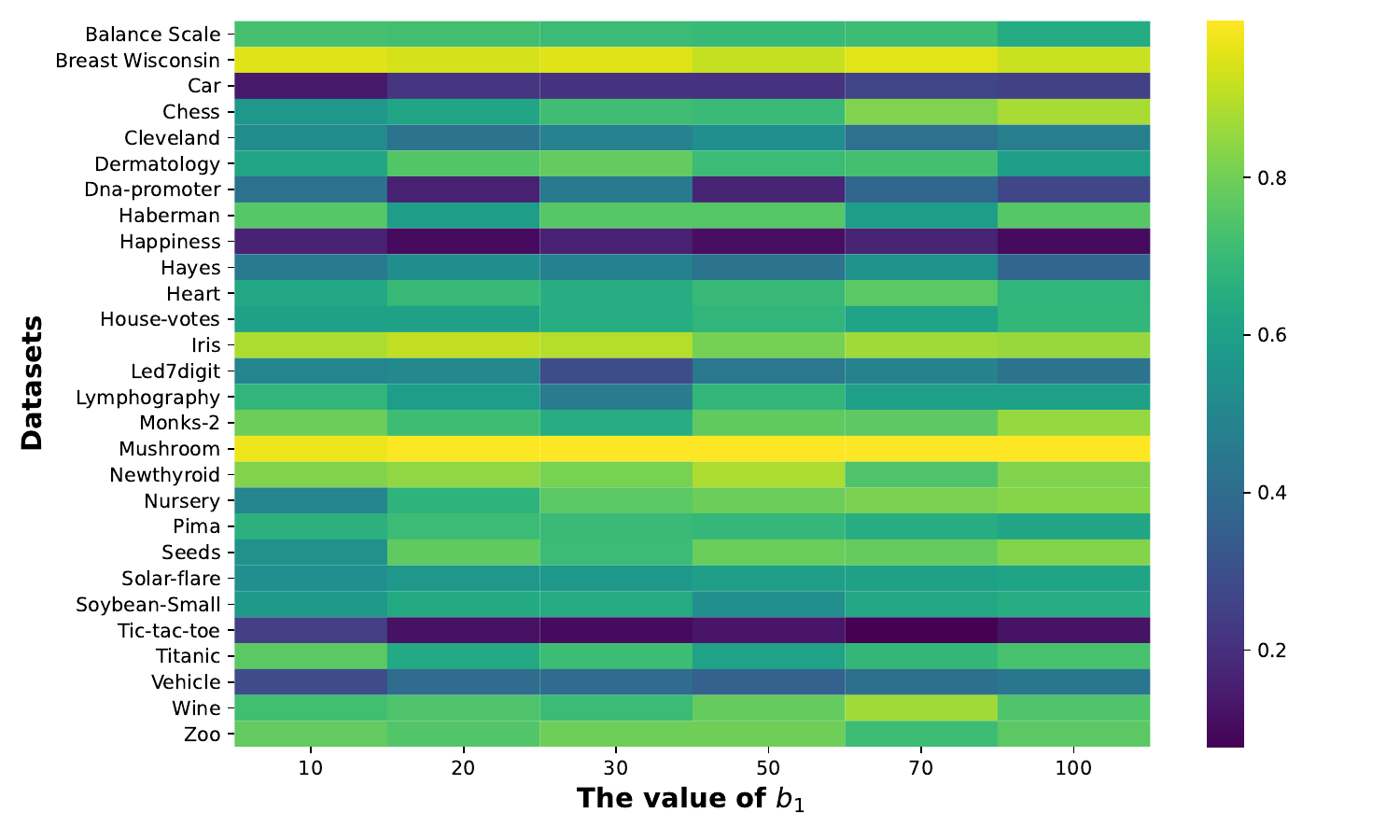}
	\caption{The effect of the value of $b_1$ on COST in terms of the JacAcc score.}
	\label{figure:b1JA}
\end{figure}

\begin{figure}[!htbp]
	\centering
	\includegraphics[width=\columnwidth]{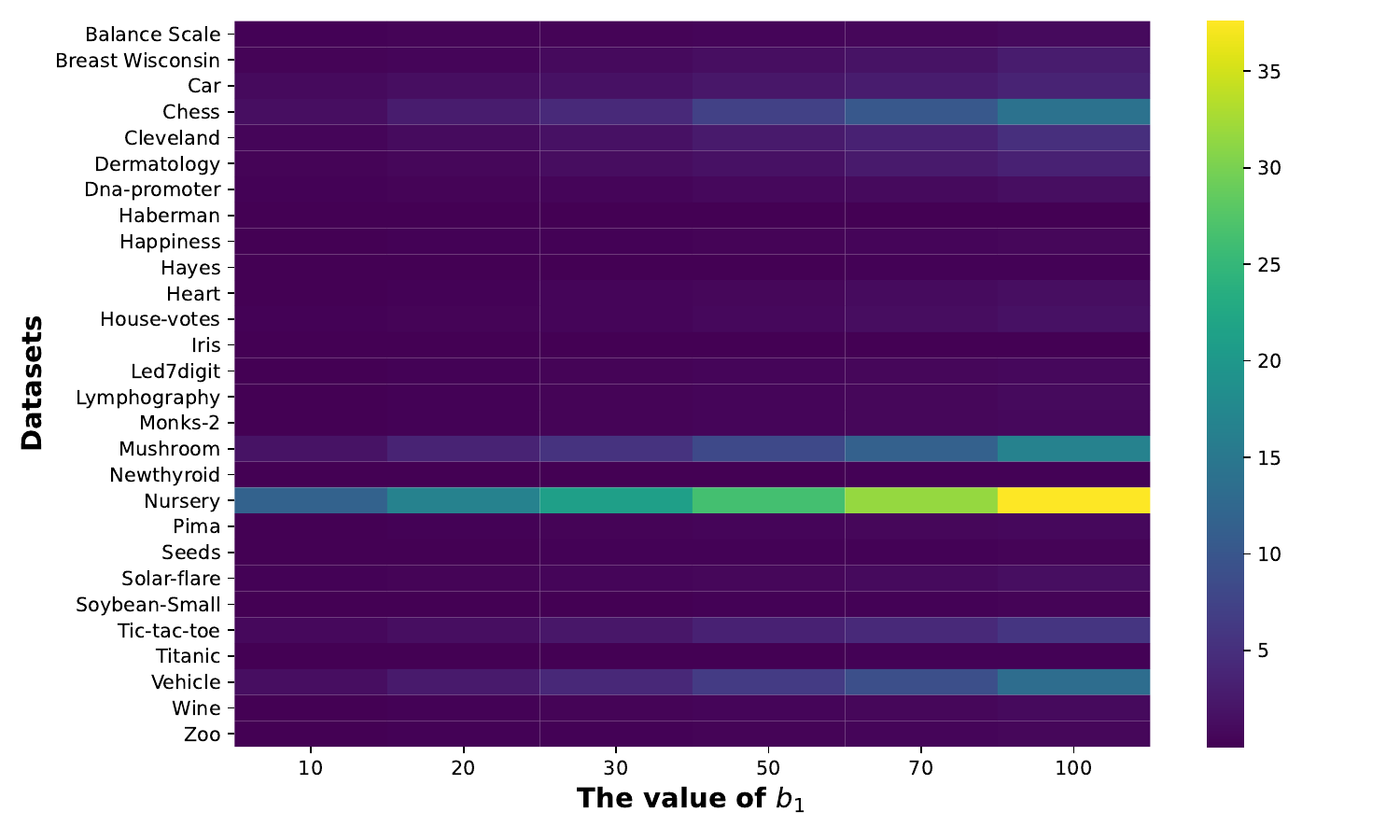}
	\caption{The effect of the value of $b_1$ on COST in terms of the running time.}
	\label{figure:b1Time}
\end{figure}

\begin{figure}[!htbp]
	\centering
	\includegraphics[width=\columnwidth]{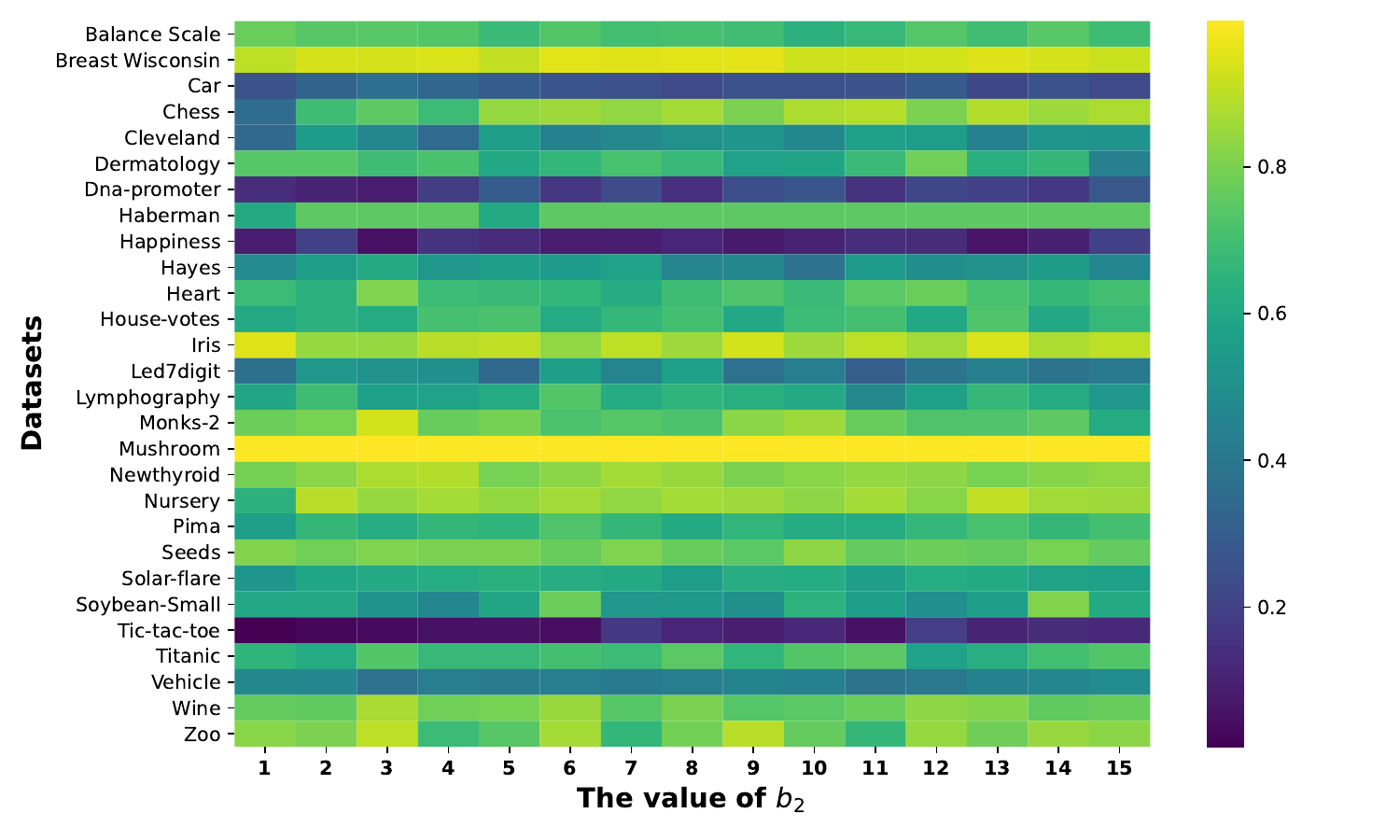}
	\caption{The effect of the value of $b_2$ on COST in terms of the JacAcc score.}
	\label{figure:b2JA}
\end{figure}

\begin{figure}[!htbp]
	\centering
	\includegraphics[width=\columnwidth]{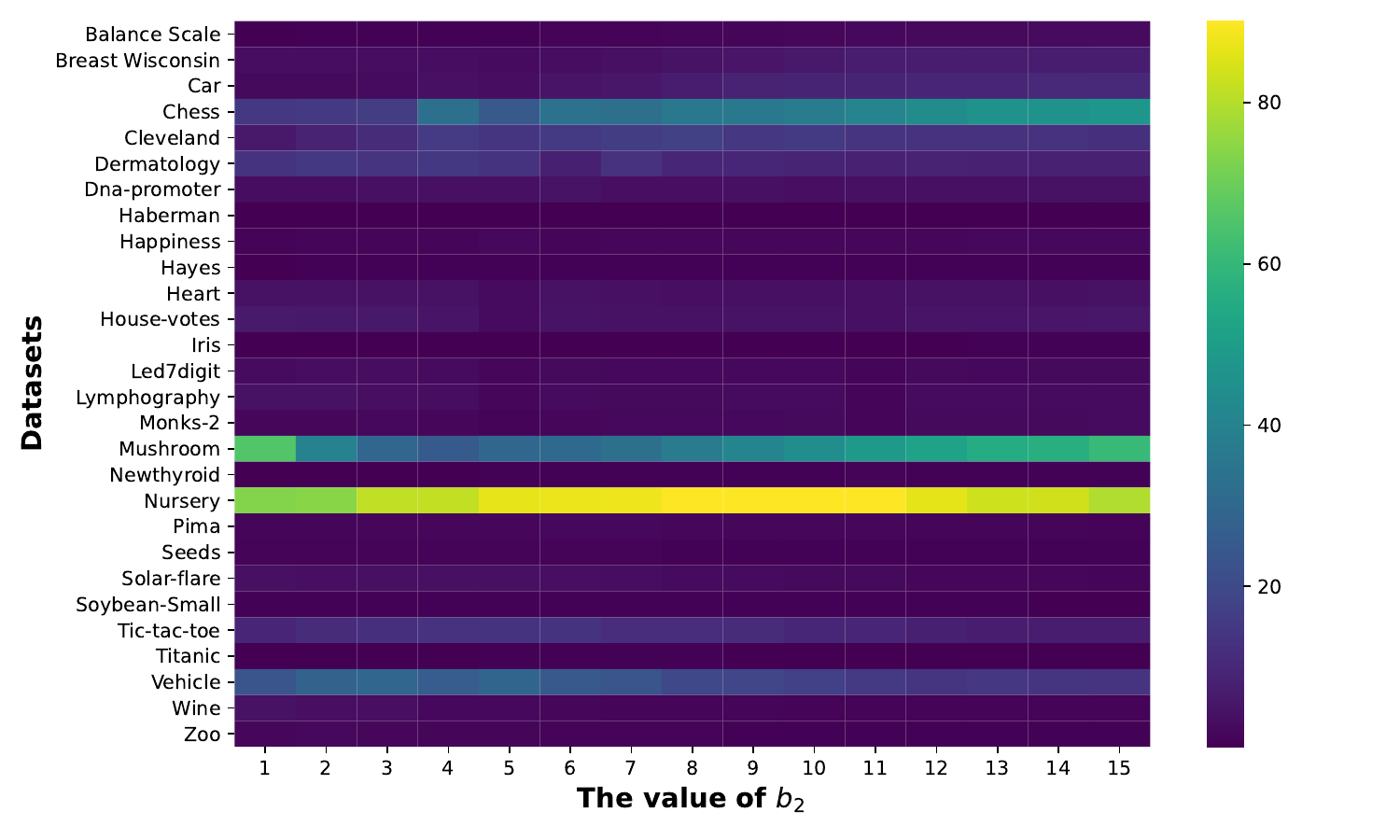}
	\caption{The effect of the value of $b_2$ on COST in terms of the running time.}
	\label{figure:b2Time}
\end{figure}

\begin{figure}[!htbp]
	\centering
	\includegraphics[width=\columnwidth]{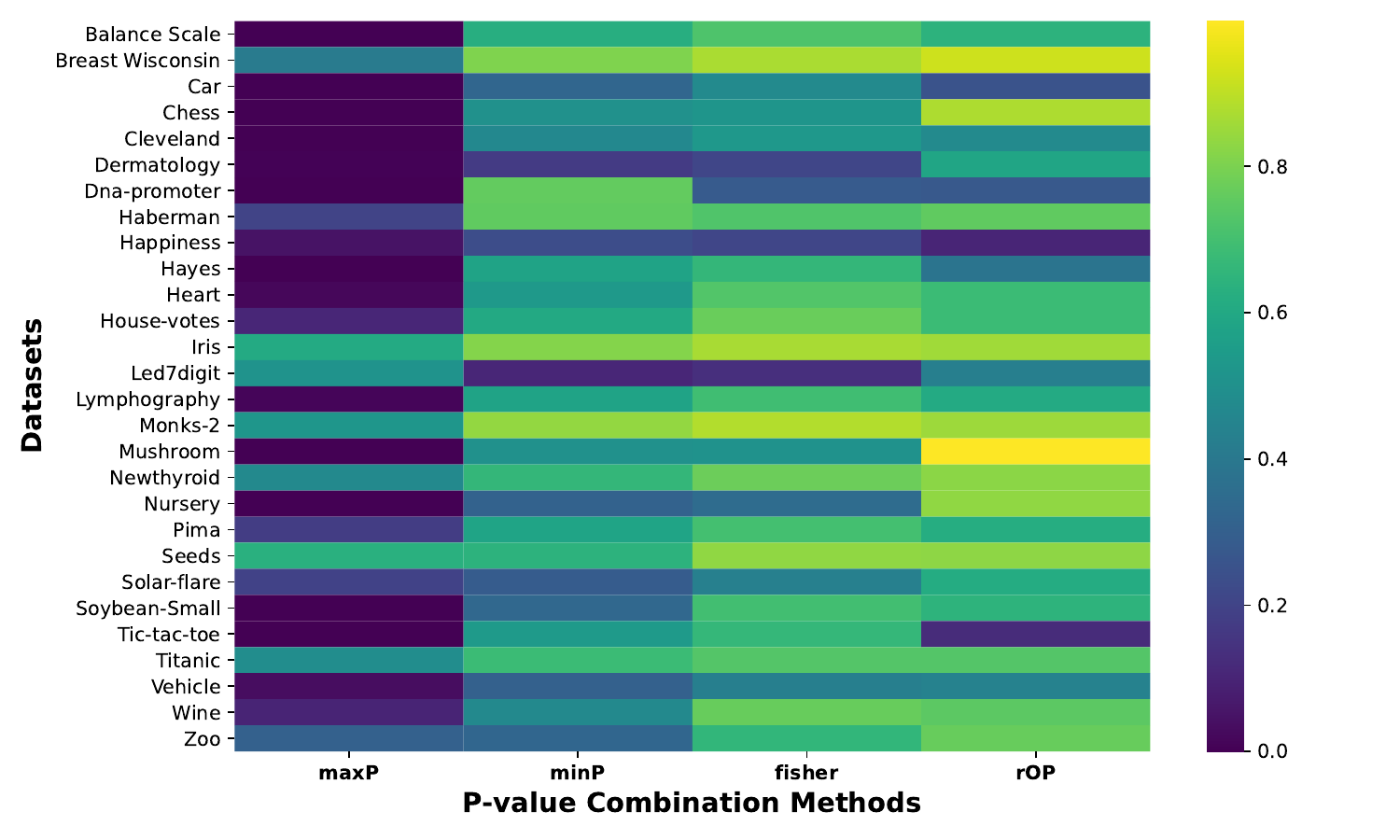}
	\caption{The effect of different \textit{p}-value combination methods on COST in terms of the JacAcc score.}
	\label{figure:combination-JA}
\end{figure}

\section{Conclusion}\label{conclusion}

In this paper, we present a new selective classification method named COST, which is a versatile classifier by combining ideas from different domains. Essentially, it is a testing-based classifier by combining significance testing results from multiple randomly chosen subspaces. The consensus \textit{p}-value for each class can be easily deployed for the purpose of both conformal and selective classification. Extensive empirical studies are conducted to demonstrate its effectiveness in different types of classification tasks. 

\section*{Acknowledgements}
This work has been partially supported by the Natural Science Foundation of China under Grant No. 61972066. We are grateful to the authors of the GPS \cite{zhou2023JMLR} algorithm, who provided invaluable assistance and timely responses to us on testing the GPS algorithm.

\bibliography{sn-bibliography}% common bib file
%% if required, the content of .bbl file can be included here once bbl is generated
%%\input sn-article.bbl

\end{document}